\documentclass{article}

\usepackage[utf8]{inputenc}
\usepackage[english]{babel}
\usepackage[T1]{fontenc}

\usepackage{amsthm}
\usepackage{mathrsfs}
\usepackage{fancyhdr}
\usepackage[colorlinks,linkcolor=red,citecolor=blue,urlcolor=black,filecolor=blue,backref=page]{hyperref}
\usepackage{graphicx}
\usepackage{amsmath}
\usepackage{amssymb}
\usepackage{dsfont}

\usepackage{algorithm}
\usepackage{algpseudocode}

\usepackage{natbib}
\usepackage{enumitem}  

\usepackage{authblk}
\usepackage{fullpage}

\theoremstyle{plain}
\newtheorem{theorem}{Theorem}[section]
\newtheorem{lemma}[theorem]{Lemma}

\newtheorem{proposition}[theorem]{Proposition}

\theoremstyle{remark}

\newcommand{\prob}{\mathbb{P}}
\newcommand{\mean}{\mathbb{E}}
\newcommand{\med}{\text{med}}
\newcommand{\Var}{\mathbb{V}}

\newcommand{\ud}{\mathop{}\!\mathrm{d}}
\newcommand{\cond}{\,|\,}

\newcommand{\e}{e} 

\DeclareMathOperator{\Normal}{\mathcal{N}}

\DeclareMathOperator{\Poi}{{Poi}}

\DeclareMathOperator{\GP}{\mathcal{GP}}
\DeclareMathOperator{\Unif}{\mathcal{U}}

\newcommand{\Bb}{\mathbf{b}}

\newcommand{\Bh}{\mathbf{h}}
\newcommand{\Bk}{\mathbf{k}}

\newcommand{\Bx}{\mathbf{x}}
\newcommand{\By}{\mathbf{y}}

\newcommand{\BA}{\mathbf{A}}
\newcommand{\BB}{\mathbf{B}}
\newcommand{\BC}{\mathbf{C}}
\newcommand{\BD}{\mathbf{D}}
\newcommand{\BH}{\mathbf{H}}
\newcommand{\BK}{\mathbf{K}}

\newcommand{\BP}{\mathbf{P}}
\newcommand{\BR}{\mathbf{R}}
\newcommand{\BS}{\mathbf{S}}

\newcommand{\BX}{\mathbf{X}}
\newcommand{\BY}{\mathbf{Y}}

\newcommand{\BW}{\mathbf{W}}
\newcommand{\BSigma}{\boldsymbol{\Sigma}}

\newcommand{\Bmu}{\boldsymbol{\mu}}
\newcommand{\Btheta}{\boldsymbol{\theta}}
\newcommand{\Bvartheta}{\boldsymbol{\vartheta}}

\newcommand{\BLambda}{\boldsymbol{\Lambda}}
\newcommand{\Bphi}{\boldsymbol{\phi}}

\renewcommand{\epsilon}{\varepsilon}
\newcommand{\Bgamma}{\boldsymbol{\gamma}}

\newcommand{\reals}{\mathbb{R}}

\newcommand*{\half}[1][2]{\frac{1}{#1}}

\newcommand{\Bzeros}{\mathbf{0}}

\newcommand{\myquad}{\quad\quad}

\newcommand{\Deltav}{\tau^2}
\newcommand{\lik}{f}
\newcommand{\eqdef}{\triangleq}
\newcommand{\T}{^{\top}}
\newcommand{\diag}{\operatorname{diag}}

\newcommand{\slmcmc}{SL-MCMC}
\newcommand{\iqr}{\textnormal{IQR}}
\newcommand{\maxv}{\textnormal{MAXV}}
\newcommand{\maxiqr}{\textnormal{MAXIQR}}
\newcommand{\eiv}{\textnormal{EIV}}
\newcommand{\miiqr}{\textnormal{IMIQR}}
\newcommand{\eiiqr}{\textnormal{EIIQR}}

\newcommand{\indic}{\mathds{1}}
\newcommand{\simiid}{\overset{\text{i.i.d.}}{\sim}}

\newcommand{\cm}{\tilde{m}}
\newcommand{\cn}{\tilde{c}}
\newcommand{\cs}{\tilde{s}}
\newcommand{\bu}{\bullet}

\newcommand{\probmeas}{\Pi}

\newcommand{\piabc}{\pi_{\text{ABC}}}

\newcommand{\tildepiapprox}{\tilde{\pi}}
\newcommand{\tilded}{\tilde{d}}
\newcommand{\piapprox}{\pi}

\newcommand{\xt}{t} 
\newcommand{\xtb}{t+b} 
\newcommand{\iter}{i} 

\def\app#1#2{%
  \mathrel{%
    \setbox0=\hbox{$#1\sim$}%
    \setbox2=\hbox{%
      \rlap{\hbox{$#1\propto$}}%
      \lower1.1\ht0\box0%
    }%
    \raise0.25\ht2\box2%
  }%
}

\allowdisplaybreaks 

\newcommand{\appe}{\textnormal{Appendix}}

\usepackage{pict2e}
\makeatletter
\DeclareRobustCommand{\bigtimes}{%
  \mathop{\vphantom{\sum}\mathpalette\@bigtimes\relax}\slimits@
}
\newcommand{\@bigtimes}[2]{\vcenter{\hbox{\make@bigtimes{#1}}}}
\newcommand{\make@bigtimes}[1]{%
  \sbox\z@{$\m@th#1\sum$}%
  \setlength{\unitlength}{\wd\z@}%
  \begin{picture}(1,1)
  \linethickness{.17ex}
  \Line(.1,.1)(.9,.9)
  \Line(.1,.9)(.9,.1)
  \end{picture}%
}
\makeatother

\title{Parallel Gaussian process surrogate Bayesian inference with noisy likelihood evaluations}

\author[1]{Marko J\"{a}rvenp\"{a}\"{a}}
\author[2]{Michael U. Gutmann}
\author[1]{Aki Vehtari}
\author[1]{Pekka Marttinen}

\affil[1]{Helsinki Institute for Information Technology HIIT, Department of Computer Science, Aalto University}
\affil[2]{School of Informatics, University of Edinburgh}

\begin{document}

\maketitle

\begin{abstract}
We consider Bayesian inference when only a limited number of noisy log-likelihood evaluations can be obtained. This occurs for example when complex simulator-based statistical models are fitted to data, and synthetic likelihood (SL) method is used to form the noisy log-likelihood estimates using computationally costly forward simulations. We frame the inference task as a sequential Bayesian experimental design problem, where the log-likelihood function is modelled with a hierarchical Gaussian process (GP) surrogate model, which is used to efficiently select additional log-likelihood evaluation locations. Motivated by recent progress in the related problem of batch Bayesian optimisation, we develop various batch-sequential design strategies which allow to run some of the potentially costly simulations in parallel. We analyse the properties of the resulting method theoretically and empirically. Experiments with several toy problems and simulation models suggest that our method is robust, highly parallelisable, and sample-efficient. 

\end{abstract}

\section{Introduction} \label{sec:intro}

When the analytic form of the likelihood function of a statistical model is available, standard sampling techniques such as Markov Chain Monte Carlo (MCMC, see e.g.~\citet{Robert2004}) can often be used for Bayesian inference. However, many models of interest in several areas of science, for example in computational biology and ecology, have an expensive-to-evaluate or intractable likelihood function which severely complicates inference. When the likelihood is intractable but forward simulation of the model is feasible, simulation-based inference methods (also called likelihood-free inference) such as approximate Bayesian computation (ABC) can be used. Unfortunately, such algorithms typically require a huge number of simulations making inference computationally costly. Examples of models with intractable likelihoods can be found in e.g.~\citet{Beaumont2002,Marin2012,Lintusaari2016,Marttinen2015,Jarvenpaa2018_aoas} and Section \ref{sec:real_example} of this article.

Surrogate models, also called meta-models or emulators, such as Gaussian processes \citep{Rasmussen2006} have been used extensively to calibrate deterministic computer codes, see e.g.~\cite{Kennedy2001}. 
GP surrogates have recently also been used to accelerate Bayesian inference by modelling some part of the inferential process, such as the log-likelihood function. The model allows extracting information from the simulations efficiently, and can be used e.g.~to determine where additional simulations are needed. For example, \citet{Rasmussen2003,Kandasamy2015,Sinsbeck2017,Drovandi2018,Wang2018,Acerbi2018} have developed GP-based techniques to accelerate Bayesian inference when the exact likelihood or the corresponding deterministic model is tractable but expensive. 
%
Various GP surrogate techniques have been proposed also for ABC, where one can only draw samples i.e.~pseudo-data from a statistical model but not evaluate the likelihood. These include \citet{Meeds2014,Jabot2014,Wilkinson2014,Gutmann2016,Jarvenpaa2018_acq}.

In this paper we focus on GP surrogate modelling of noisy log-likelihood evaluations. 
Earlier works on emulating the log-likelihood function have mostly assumed exact, i.e., noiseless evaluations or the noise has not been explicitly modelled. We show that noisy evaluations cause extra challenges. Also, although not the focus of this work, we remark that one often has some control over the noise level. While our approach is applicable whenever noisy, expensive log-likelihood evaluations of a statistical model of interest are available, we mainly focus on likelihood-free inference using the synthetic likelihood method \citep{Wood2010,Price2018}, where the intractable log-likelihood is approximated using repeated forward simulations at each evaluation location. 

Recently, \citet{Jarvenpaa2018_acq} developed a Bayesian decision theoretic framework for ABC inference and considered sequential strategies (also called active learning) to select the next evaluation location for an expensive simulation model. Here we extend this framework in two ways: 1) we modify it to address the problem of Bayesian inference using noisy log-likelihood evaluations, which is different from ABC, and 2) we develop batch-sequential design strategies to efficiently parallelise the estimation of the surrogate likelihood. 
In earlier related works the simulation locations have been selected either sequentially \citep{Kandasamy2015,Sinsbeck2017,Wang2018,Acerbi2018,Jarvenpaa2018_acq} or using simple heuristics \citep{Wilkinson2014,Gutmann2016}. Batch strategies are useful when a computing cluster is available and, as we show, can substantially reduce the computation time compared to the corresponding sequential strategies. We also analyse some properties of the proposed methods theoretically, and conduct an extensive empirical comparison. 

Our approach is closely related to Bayesian quadrature (BQ), see e.g.~\citet{OHagan1991,Hennig2015,Karvonen2018}. In particular, BQ methods have been used by \citet{Osborne2012,Gunter2014,Chai2018} to compute the marginal likelihood and to quantify the numerical error of this integral probabilistically. In this article we are not interested in this particular integral but in obtaining an accurate estimate of the posterior. Also, we allow noisy log-likelihood evaluations. 
%
Another related problem is Bayesian optimisation (BO), see e.g.~\citet{Brochu2010,Shahriari2015}. 
Our objective to parallelise simulations is motivated by 
recent research on batch Bayesian optimisation \citep{Ginsbourger2010,Azimi2010,Snoek2012,Contal2013,
Desautels2014,Shah2015,Wu2016,Gonzalez2016, Wilson2018}. 
However, while BO methods can be used to accelerate likelihood-free Bayesian inference \citep{Gutmann2016}, they are not specifically designed for estimating the posterior (see discussion in e.g.~\citet{Kandasamy2015,Jarvenpaa2018_acq}).
Similarly to \citet{Jarvenpaa2018_acq}, we explicitly design our algorithms from the first principles of Bayesian decision theory to acknowledge the goal of the analysis, i.e. estimation of the posterior density. 
%
Finally, we note that GPs in conjuction with Bayesian experimental designs have also been successful in estimating level and excursion sets of expensive-to-evaluate functions, see e.g.~\citet{Bect2012,Chevalier2014,Lyu2018}. 

This paper is organised as follows. Section \ref{sec:abc_and_sl} briefly reviews ABC and the SL. Sections \ref{sec:gp_model} and \ref{sec:estimators} contain the details of our GP surrogate model and posterior estimation. Batch-sequential design strategies for sample-efficient estimation of the (approximate) posterior distribution are developed in Section \ref{sec:acq} while Section \ref{sec:experiments} contains experiments. Finally, Section \ref{sec:discussion} contains discussion and concluding remarks. Proofs, implementation details and additional experiments can be found in the \appe{}.

\section{ABC and the synthetic likelihood methods}\label{sec:abc_and_sl}

Our goal is to estimate parameters $\Btheta\in\Theta$ of a simulation model given observed data $\Bx\in\mathcal{X}$. We assume $\Theta$ is a compact subset of $\reals^d$ and that the prior information about feasible values of $\Btheta$ is coded into a (continuous) prior pdf $\pi(\Btheta)$. For simplicity we consider only continuous parameters but discrete parameters can be handled similarly. If evaluating the likelihood function $\pi(\Bx\cond\Btheta)$ is feasible, the posterior distribution can be computed using Bayes' theorem $\pi(\Btheta \cond \Bx) \propto \pi(\Btheta)\pi(\Bx\cond\Btheta)$ up to a normalisation constant and hence be used as a target distribution in MCMC. 
However, when the likelihood is too costly to evaluate or unavailable, standard MCMC algorithms become infeasible. 


Even when the likelihood is intractable, simulating ``pseudo-data'' from the model, i.e., drawing samples $\Bx_{\Btheta} \sim \pi(\cdot\cond\Btheta)$, is often feasible. In this case, ABC can be used for inference, see e.g.~\citet{Marin2012,Turner2012,Lintusaari2016}. Standard ABC techniques approximate the posterior as
\begin{align}
\piabc(\Btheta \cond \Bx) \propto \pi(\Btheta)\int_{\mathcal{X}} \indic_{\Delta(\Bx,\Bx_s) \leq \epsilon} \pi(\Bx_s\cond\Btheta) \ud \Bx_s, \label{eq:abc_post}
\end{align}
where $\indic$ denotes the indicator function, $\epsilon$ is a tolerance parameter and $\Delta: \mathcal{X}^2 \rightarrow \reals_{+}$ is the discrepancy function  used to compute the similarity between the simulated data $\Bx_s$ and the observed data $\Bx$. The discrepancy is typically constructed from  low-dimensional summary statistics $S:\mathcal{X} \rightarrow \reals^p$, so that~$\Delta(\Bx,\Bx_s) = \Delta'(S(\Bx),S(\Bx_s))$, where $\Delta':\reals^p\times\reals^p \rightarrow \reals_{+}$ is, for example, the weighted Euclidean distance. 
For each proposed parameter $\Btheta$, an unbiased ABC posterior estimate can be obtained by replacing the integral in \eqref{eq:abc_post} with a Monte Carlo sum using some $N$ simulated pseudo-data sets $\Bx^{(i)}_{\Btheta} \sim \pi(\cdot\cond\Btheta)$ for $i=1,\ldots,N$.

An alternative to ABC is the synthetic likelihood method \citep{Wood2010,Price2018}. In SL the summary statistics $S(\Bx_{\Btheta})$ are assumed to have a Gaussian distribution for each parameter $\Btheta$, that is 
%
\begin{equation}
\pi(\Bx\cond\Btheta) 
\approx \pi(S(\Bx) \cond \Btheta) 
\approx \pi_{\text{SL}}(S(\Bx) \cond \Btheta) 
\eqdef \Normal(S(\Bx)\cond\Bmu_{\Btheta},\BSigma_{\Btheta}). 
%
\label{eq:SL}
\end{equation}
The first approximation 
results from replacing the full data $\Bx$ with a potentially nonsufficient summary statistics $S(\Bx)$. The second approximation is due to the possible violations of the Gaussianity of $S(\Bx)$.
The expectation $\Bmu_{\Btheta}$ and covariance matrix $\BSigma_{\Btheta}$ in \eqref{eq:SL} are unknown and are estimated for each proposed parameter $\Btheta$ using maximum likelihood (ML):
%
\begin{equation}
\hat{\Bmu}_{\Btheta} = \frac{1}{N}\sum_{i=1}^{N}S(\Bx^{(i)}_{\Btheta}), 
%
\quad \hat{\BSigma}_{\Btheta} = \frac{1}{N-1}\sum_{i=1}^{N} (S(\Bx^{(i)}_{\Btheta}) - \hat{\Bmu}_{\Btheta}) (S(\Bx^{(i)}_{\Btheta}) - \hat{\Bmu}_{\Btheta})\T, \label{eq:SL_ML}
\end{equation}
where $\Bx^{(i)}_{\Btheta} \sim \pi(\cdot\cond\Btheta)$ for $i=1,\ldots,N$. 
As investigated by \citet{Price2018}, the standard Metropolis algorithm can be combined with SL. The likelihood is then computed using \eqref{eq:SL} and the ML estimates in \eqref{eq:SL_ML} or, alternatively, using an unbiased estimate of $\Normal(S(\Bx)\cond\Bmu_{\Btheta},\BSigma_{\Btheta})$ shown in Section 2.1 of \citet{Price2018} which produces an exact pseudo-marginal MCMC if the Gaussianity assumption holds. 
See \appe{} \ref{app:add_details} for discussion on the use of different SL estimators.  

The advantage of SL over ABC is that specifying suitable ABC tuning parameters such as the tolerance and the discrepancy is avoided. While the Gaussianity of the summary statistics may not hold in practice, \citet{Price2018} have found that SL is often robust to deviations from normality. 
%
%
SL and its extensions \citep{Thomas2018,An2018,An2019,Nott2019} produce pointwise noisy log-likelihood evaluations because in practice the number of repeated simulations $N$ at each point is finite. Using (pseudo-marginal) MCMC or other sampling-based techniques for inference with these noisy targets thus requires a large number of simulations. Assuming noisy log-likelihood evaluations are available, e.g.~obtained by using SL, the goal of the following sections is to develop an inference algorithm that can minimise the number of evaluations needed.



\section{Gaussian process surrogate for the noisy log-likelihood}
\label{sec:gp_model}

We denote the log-likelihood or its approximation, such as the log-SL obtained as the logarithm of \eqref{eq:SL}, as $\lik(\Btheta) \eqdef \log \pi(\Bx\cond\Btheta)$. We assume that we have access to noisy log-likelihood evaluations at $\Btheta_i$ denoted by $y_i\in\reals$ for building the surrogate model and that the ``noise'' i.e.~the numerical or sampling error in evaluating the log-likelihood is independently Gaussian distributed. Treating the noisy log-likelihood evaluations $y_i$ as ``observations'', our measurement model is
\begin{equation}
    y_i = \lik(\Btheta_i) + \sigma_n(\Btheta_i)\epsilon_i, \quad \epsilon_i \simiid \Normal(0,1),
    \label{eq:noisemodel}
\end{equation}
where $\sigma_n: \Theta \rightarrow (0,\infty)$ is a (continuous) function of $\Btheta$ that determines the standard deviation of the observation noise and is assumed known. 
To justify our model in \eqref{eq:noisemodel}, we show empirically that log-SL is well approximated by a Gaussian distribution using six benchmark simulation models in the \appe{} \ref{subsec:logSL_normality}.

We place the following hierarchical GP prior for the log-likelihood function $\lik$:
\begin{align}\begin{split}
    \lik \cond \Bgamma \sim \GP(m_0(\Btheta),k(\Btheta,\Btheta')), \quad 
    m_0(\Btheta) = \sum_{i=1}^q \gamma_i h_i(\Btheta), \quad \Bgamma \sim \Normal(\Bb,\BB), \label{eq:gp_prior}
\end{split}\end{align}
where $k:\Theta^2\rightarrow\reals$ is a covariance function and $h_i:\Theta\rightarrow\reals$ are fixed basis functions (both assumed continuous). 
%
The nuisance parameters $\Bgamma$ in \eqref{eq:gp_prior} are marginalised, see e.g.~\citet{OHagan1978,Rasmussen2006}, 
to obtain the following equivalent GP prior
\begin{equation}
    \lik \sim \GP(\Bh(\Btheta)\T\Bb,k(\Btheta,\Btheta') + \Bh(\Btheta)\T\BB \Bh(\Btheta')),
\end{equation}
where $\Bh(\Btheta)\in\reals^q$ is a column vector consisting of the basis functions $h_i$ evaluated at $\Btheta$. We use basis functions of the form $1, \theta_i, \theta_i^2$. 
A similar GP prior has been considered in \citet{Wilkinson2014,Gutmann2016,Drovandi2018}, however, different from those articles, we take a fully Bayesian approach and marginalise $\Bgamma$ as in \citet{Riihimaki2014}. 
%
Since little initial information is typically available on the magnitude and shape of the log-likelihood, we use relatively uninformative hyperpriors so that $\Bb=\Bzeros$ and $B_{ij}=30^2\indic_{i=j}$. We assume that the log-likelihood function is smooth, and adopt the squared exponential covariance function $k(\Btheta,\Btheta') = \sigma_{\lik}^2 \exp(-\sum_{i=1}^d(\theta_i-\theta_i')^2/(2l_i^2))$ although other choices, such as the Mat\'{e}rn covariance function, are also possible. We denote the $d+1$ covariance function hyperparameters as $\Bphi = (\sigma_{\lik}^2,l_1,\ldots,l_d)$. For now, we assume $\Bphi$ is known and omit it from our notation for simplicity. 

Given observations $D_{\xt}=\{(y_i,\Btheta_i)\}_{i=1}^t$, which we call training data, our knowledge of the log-likelihood function is $\lik \cond D_{\xt} \sim \GP(m_{\xt}(\Btheta),c_{\xt}(\Btheta, \Btheta'))$, where 
\begin{align}
    m_{\xt}(\Btheta) &\eqdef k_{\xt}(\Btheta) \BK^{-1}_{\xt} \By_{\xt} 
    + \BR_{\xt}\T(\Btheta) \bar{\Bgamma}_{\xt}, \label{eq:gp_mean} \\ 
    %
    \begin{split}
    c_{\xt}(\Btheta,\Btheta') 
    &\eqdef k(\Btheta,\Btheta') 
    \!-\! k_{\xt}(\Btheta) \BK^{-1}_{\xt} k\T_{\xt}(\Btheta') 
    \!+\!\BR_{\xt}\T(\Btheta)[\BB^{-1}\!+\!\BH_{\xt}\BK^{-1}_{\xt}\BH_{\xt}\T]^{-1} \BR_{\xt}(\Btheta'), 
    \end{split} \label{eq:gp_cov}
\end{align}
with $[\BK_{\xt}]_{ij} \eqdef k(\Btheta_{i},\Btheta_{j}) + \indic_{i=j}\sigma_n^2(\Btheta_i)$ for $i,j\in\{1,\ldots,t\}$, $k_{\xt}(\Btheta) \eqdef (k(\Btheta,\Btheta_1),\ldots,k(\Btheta,\Btheta_t))\T$, 
\begin{align}
    \bar{\Bgamma}_{\xt} \eqdef [\BB^{-1} + \BH_{\xt}\BK^{-1}_{\xt}\BH_{\xt}\T]^{-1}(\BH_{\xt}\BK^{-1}_{\xt}\By_{\xt} + \BB^{-1}\Bb), 
    %
\end{align}
and $\BR_{\xt}(\Btheta) \eqdef \BH(\Btheta) - \BH_{\xt}\BK^{-1}_{\xt}k_{\xt}\T(\Btheta)$. 
Above $\BH_{\xt}$ is the $q\times t$ matrix whose columns consist of basis function values evaluated at training points $\Btheta_{1:t}=[\Btheta_1,\ldots,\Btheta_t]$ which is itself a $d\times t$ matrix, and $\BH(\Btheta)$ is the corresponding $q\times 1$ vector at test point $\Btheta$. 
From now on, we denote the GP variance function as $s_{\xt}^2(\Btheta) \eqdef c_{\xt}(\Btheta,\Btheta)$ and the probability law of $\lik$ given $D_{\xt}$ as $\probmeas_{D_{\xt}}^{\lik}$, that is, $\probmeas_{D_{\xt}}^{\lik} \eqdef \GP(m_{\xt}(\Btheta),c_{\xt}(\Btheta, \Btheta'))$.

\section{Estimators of the posterior from the GP surrogate} \label{sec:estimators}

Using the GP surrogate model for the noisy log-likelihood, we here derive estimators for the posterior which can be e.g.~plugged-in to a MCMC algorithm. Resulting sampling algorithms do not require further simulator runs (unlike e.g.~\slmcmc{}) producing potentially huge computational savings. Figure~\ref{fig:loglik_gp} demonstrates our approach. 
We want to use our knowledge of the log-likelihood function represented by $\probmeas_{D_{\xt}}^{\lik}$ to determine the optimal point estimate of the probability density function (pdf) of the posterior\footnote{While in this article we are mainly concerned with point estimators of the posterior pdf, we can also quantify its (epistemic) uncertainty similarly to \emph{probabilistic numerics} literature (see e.g.~\citet{Hennig2015,Cockayne2017,Briol2019}) as illustrated in Figure \ref{fig:loglik_gp}b. Such uncertainty estimates are also used to intelligently select the next simulation locations in Section \ref{sec:acq}.}. 
The uncertainty of the log-likelihood $\lik$ can be propagated to the posterior distribution of the simulation model which consequently becomes a random quantity denoted as $\piapprox_{\lik}$:
\begin{equation}
\piapprox_{\lik}(\Btheta) 
\eqdef \frac{\pi(\Btheta)\exp(\lik(\Btheta))}{\int_{\Theta}\pi(\Btheta')\exp(\lik(\Btheta'))\ud\Btheta'}. 
\label{eq:quantity_of_interest}
\end{equation}
The expectation of the posterior pdf $\piapprox_{\lik}$ at each parameter $\Btheta$ can be formally written as
\begin{align}
\mean_{\lik\cond D_{\xt}} (\piapprox_{\lik}(\Btheta))
=
%
\int\!\frac{\pi(\Btheta)\exp(\lik(\Btheta))}{\int_{\Theta}\pi(\Btheta')\exp(\lik(\Btheta'))\ud\Btheta'} \probmeas_{D_{\xt}}^{\lik}(\ud\lik)
\label{eq:e_norm_post}
\end{align}
and the variance can be obtained similarly (assuming these quantities exist). In principle, one could sample posterior pdfs by first drawing $\lik^{(i)}\sim\probmeas_{D_{\xt}}^{\lik}$ (a continuous function $\Theta\rightarrow\reals$), then computing $\pi(\Btheta)\exp(\lik^{(i)}(\Btheta))$, and finally normalising.
However, in practice this would require discretisation of the $\Theta$-space and involves computational challenges. For this reason and similarly to \citet{Sinsbeck2017,Jarvenpaa2018_acq}, we instead take our quantity of interest to be the unnormalised posterior 
\begin{equation}
\tildepiapprox_{\lik}(\Btheta) \eqdef \pi(\Btheta)\exp(\lik(\Btheta)), \label{eq:unn_post}
\end{equation}
which follows log-Gaussian process that allows analytical computations.

\begin{figure*}[ht]
\centering
\includegraphics[width=0.8\textwidth]{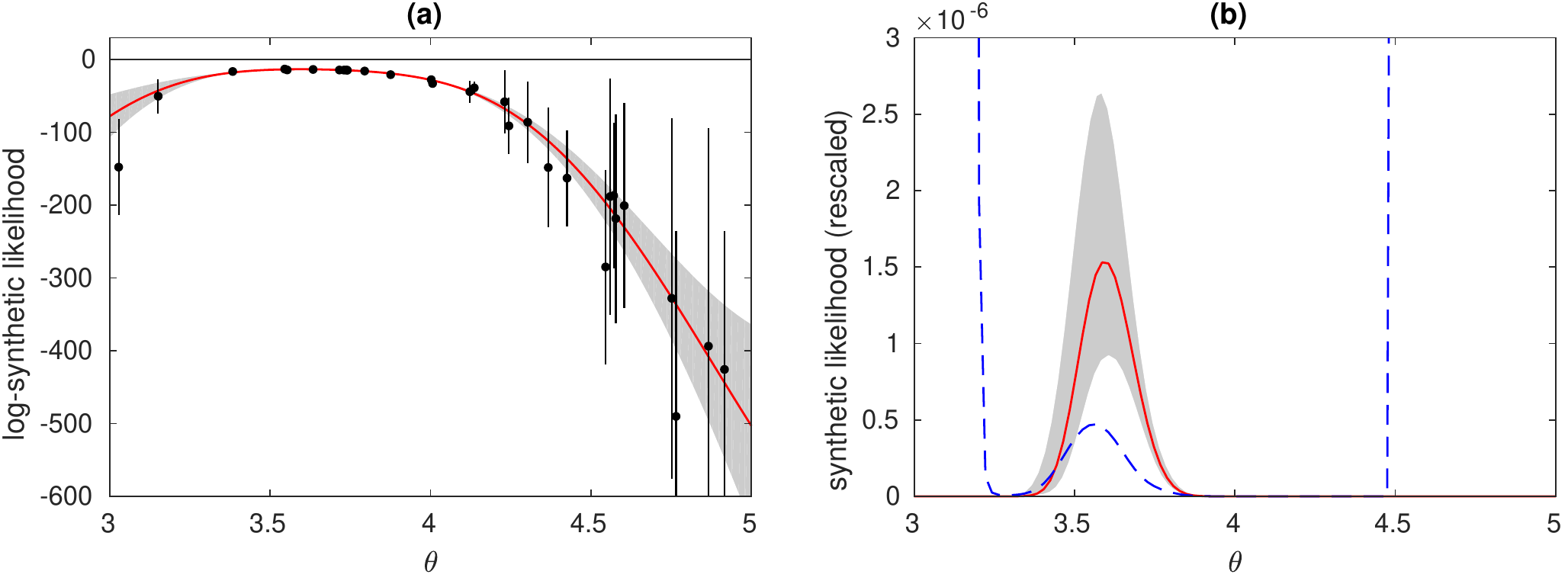}
\caption{(a) GP surrogate model for the log-SL of the Ricker model of Section \ref{sec:ricker} when only the first parameter, $\theta=\log(r)$, is varied. The black dots show the noisy log-SL evaluations and the black lines their approximate 95\% confidence intervals, the grey area the 95\% credible interval, and the red line the GP mean function. (b) Uncertainty of the SL. The grey area shows the 95\% credible interval of the SL and the red line is the median estimate, obtained from \eqref{eq:log_med_q}. The dashed blue line shows the standard deviation of SL computed as the square root of \eqref{eq:log_mean_var}.}
\label{fig:loglik_gp}
\end{figure*}

Next we derive an optimal estimator for the unnormalised posterior $\tildepiapprox$ in \eqref{eq:unn_post} using Bayesian decision-theory. We proceed here similarly to \citet{Sinsbeck2017} and consider the integrated quadratic loss function $l_2(\tildepiapprox_1,\tildepiapprox_2) \eqdef \int_{\Theta} (\tildepiapprox_1(\Btheta) - \tildepiapprox_2(\Btheta))^2 \ud\Btheta$ between two (unnormalised) posterior densities $\tildepiapprox_1$ and $\tildepiapprox_2$. We assume $\tildepiapprox_1$ and $\tildepiapprox_2$ are square-integrable functions in $\Theta$, i.e.~$\tildepiapprox_1,\tildepiapprox_2\in L^2(\Theta)$. The optimal Bayes estimator, denoted by $\hat{\tildepiapprox}\in\mathbb{D}$, is the minimiser of the expected loss, where $\mathbb{D}=L^2(\Theta)$ denotes the set of candidate estimators. In detail,
\begin{align}\begin{split}
\hat{\tildepiapprox} &= \arg\min_{\tilded\in\mathbb{D}} \mean_{\lik\cond D_{\xt}} l_2(\tildepiapprox_{\lik},\tilded)
= \arg\min_{\tilded\in\mathbb{D}} \mean_{\lik\cond D_{\xt}} \int_{\Theta} (\tildepiapprox_{\lik}(\Btheta) - \tilded(\Btheta))^2 \ud\Btheta \\
&= \arg\min_{\tilded\in\mathbb{D}} \int_{\Theta} \mean_{\lik\cond D_{\xt}}(\tildepiapprox_{\lik}(\Btheta) - \tilded(\Btheta))^2 \ud\Btheta,
\label{eq:bayes_risk_ltwo}
\end{split}\end{align}
where Tonelli theorem is used to change the order of expectation and integration and where $\tilded\in\mathbb{D}=L^2(\Theta)$ is a candidate estimator of $\tildepiapprox$.
Equation~\ref{eq:bayes_risk_ltwo} shows that the expected loss is minimised when the integrand on the second row is minimised independently for (almost) each $\Btheta\in\Theta$. 
It follows from the basic results of Bayesian decision theory (see e.g.~\citet{Robert2007}) that the minimum is obtained when $\tilded(\Btheta)=\mean_{\lik\cond D_{\xt}}(\tildepiapprox_{\lik}(\Btheta))$, i.e., the optimal estimator is the posterior expectation. The minimum value of \eqref{eq:bayes_risk_ltwo}, called Bayes risk, is the integrated variance $\int_{\Theta} \Var_{\lik\cond D_{\xt}}(\tildepiapprox_{\lik}(\Btheta)) \ud\Btheta$. The posterior expectation and variance can be computed from the log-Normal distribution as 
%
%
\begin{align}
	\mean_{\lik\!\cond\!D_{\xt}}(\tildepiapprox_{\lik}(\Btheta))\!=\!\pi(\Btheta)\e^{m_{\xt}(\Btheta)\!+\!\half s_{\xt}^2(\Btheta)}, 
%
\quad\! \Var_{\lik\!\cond\!D_{\xt}}(\tildepiapprox_{\lik}(\Btheta))\!=\!\pi^2(\Btheta)\e^{2m_{\xt}(\Btheta)\!+\! s_{\xt}^2(\Btheta)}\!\left(\!\e^{s_{\xt}^2(\Btheta)}\!-\!1\!\right)\!. \label{eq:log_mean_var}
\end{align}
%

If we instead use $L^1$ loss $l_1(\tildepiapprox_1,\tildepiapprox_2) \eqdef \int_{\Theta} |\tildepiapprox_1(\Btheta) - \tildepiapprox_2(\Btheta)| \ud\Btheta$ where $\tildepiapprox_1,\tildepiapprox_1\in L^1(\Theta)$, we can similarly show that the optimal point estimator is the marginal median. The median and the $\alpha-$quantile $q^{\alpha}$ with $\alpha\in(0,1)$ can be computed as
%
%
\begin{align}
\med_{\lik\cond D_{\xt}}(\tildepiapprox_{\lik}(\Btheta)) = \pi(\Btheta)\e^{m_{\xt}(\Btheta)}, \quad 
%
	q_{\lik\cond D_{\xt}}^{\alpha}(\tildepiapprox_{\lik}(\Btheta)) = \pi(\Btheta)\e^{m_{\xt}(\Btheta) + \Phi^{-1}(\alpha)s_{\xt}(\Btheta)}, \label{eq:log_med_q}
\end{align}
where $\Phi^{-1}$ is the quantile function of the standard normal distribution. 
As we show in the \appe{} \ref{app:proofs}, the Bayes risk corresponding to the $L^1$ loss is 
\begin{align}
\min_{\tilded\in\mathbb{D}} \mean_{\lik\cond D_{\xt}} l_1(\tildepiapprox_{\lik},\tilded) = 
    \int_{\Theta} \pi(\Btheta) \exp(m_{\xt}(\Btheta)+s^2_{\xt}(\Btheta)/2)(2\Phi(s_{\xt}(\Btheta))-1) \ud\Btheta.
    \label{eq:loneloss}
\end{align}

When MCMC is used with the point estimator of the unnormalised posterior in either \eqref{eq:log_mean_var} or \eqref{eq:log_med_q}, we are in fact targeting the following mean and median based estimators of the (normalised) posterior 
\begin{align}
\pi^{\text{mean}}_{\xt}(\Btheta) 
\eqdef \frac{\pi(\Btheta)\e^{m_{\xt}(\Btheta) + \half s_{\xt}^2(\Btheta)}}{\int_{\Theta} \pi(\Btheta')\e^{m_{\xt}(\Btheta') + \half s_{\xt}^2(\Btheta')} \ud\Btheta'}, 
%
\quad \pi^{\text{med}}_{\xt}(\Btheta) \eqdef \frac{\pi(\Btheta)\e^{m_{\xt}(\Btheta)}}{\int_{\Theta} \pi(\Btheta')\e^{m_{\xt}(\Btheta')} \ud\Btheta'}. \label{eq:point_estim_norm}
\end{align}
These are obtained by simply normalising the Bayes optimal estimators of the unnormalised posterior (and, as a consequence, a guarantee of optimality for normalised posterior does not follow). 
Similar estimators were also considered by \citet{Stuart2018}. Both are clearly valid density functions, and tractable, unlike \eqref{eq:e_norm_post}. The latter, i.e., the marginal median based estimate, is equal to \eqref{eq:quantity_of_interest} if we replace the unknown log-likelihood function $\lik(\Btheta)$ with a GP mean function $m_{\xt}(\Btheta)$ and neglect GP uncertainty. On the other hand, the former, i.e., the marginal mean estimate, takes into account the GP uncertainty through the variance function $s_{\xt}^2(\Btheta)$. These two point estimates become the same if the GP variance is negligible. 
%

\section{Parallel designs of simulations} \label{sec:acq}

In the previous section we showed how to quantify the uncertainty of the GP surrogate-based unnormalised posterior density and we derived computable and (in a certain sense) optimal point estimates of it. Next we develop Bayesian experimental design strategies to select further locations to evaluate the log-likelihood, so that the uncertainty in the unnormalised posterior decreases as fast as possible. 
%
We focus on batch strategies and denote the batch size as $b \in \{1,2,\ldots\}$.
%
Before moving on, we introduce some terminology. The next batch of $b$ evaluation locations is obtained as the solution to an optimisation problem. We call the objective function of this optimisation problem a \textit{design criterion} and the resulting batch of evaluation locations as \textit{design points} or just \textit{design}. The complete procedure of selecting the design points is called a \textit{batch-sequential} (or, when $b=1$, just \textit{sequential}) \textit{strategy}\footnote{The design criterion is often called an acquisition function and the resulting strategy sometimes an acquisition rule in the BO literature.}. 

In this paper we focus on syncronous parallelisation where a batch of $b$ design points is constructed at each iteration and the corresponding $b$ simulations are simultaneously submitted to the workers. However, the ``greedy'' design strategies developed in Sections \ref{sec:greedy_acq} and \ref{sec:other_acq} can also be used for asyncronous parallelisation, where a new location is immediately chosen and submitted for processing, whenever any of the running simulations completes, instead of waiting all the other $b-1$ simulations to finish. 

\subsection{Analytical expressions for the design criteria} \label{sec:dec}

We first derive some general results needed for efficient evaluation of the design criteria. These can be useful also for developing batch designs for other related GP-based problems such as BQ and BO. 
%
Given $D_{\xt} = \{(y_i,\Btheta_i)\}_{i=1}^t$ it is useful to know how additional $b$ candidate evaluations at points $\Btheta^* = [\Btheta^*_1,\ldots,\Btheta^*_b]$ would affect our knowledge about the log-likelihood $\lik$ and the unnormalised posterior $\tildepiapprox$. 
%
The following Lemma is central to our analysis. It shows how the GP mean and variance functions are affected by supplementing the training data $D_{\xt}$ with a new batch of evaluations $D^*=\{(y^*_i,\Btheta^*_i)\}_{i=1}^b$ when the unknown $\By^*$ is assumed to be distributed according to the posterior predictive distribution of the GP given $D_{\xt}$. The Lemma is a generalisation of a similar result by \citet{Jarvenpaa2018_acq,Lyu2018}. 
%
\begin{lemma} \label{lemma:1}
Consider the mean and variance functions of the GP model in Section \ref{sec:gp_model} for a fixed $\Btheta$, given the training data $D_{\xt} \cup D^*$ and when treated as functions of $\By^*$. Assume $\By^*$ follows the posterior predictive distribution, that is $\By^*\cond\Btheta^*,D_{\xt} \sim \Normal(m_{\xt}(\Btheta^*),c_{\xt}(\Btheta^*,\Btheta^*) + \diag(\sigma_n^2(\Btheta_1^*),\ldots,\sigma_n^2(\Btheta_b^*)))$. Then, 
\begin{align}
m_{\xtb}(\Btheta;\Btheta^*) \cond \Btheta^*,D_{\xt} &\sim \Normal(m_{\xt}(\Btheta), \Deltav_{\xt}(\Btheta;\Btheta^*)), \label{eq:gp_mean_lookahead} \\
%
s^2_{\xtb}(\Btheta;\Btheta^*) \cond \Btheta^*,D_{\xt} &\sim \delta(s^2_{\xt}(\Btheta) - \Deltav_{\xt}(\Btheta;\Btheta^*) - s^2_{\xtb}(\Btheta;\Btheta^*)), \label{eq:gp_var_lookahead} 
\end{align}
where $\delta(\cdot)$ is the Dirac measure and
\begin{equation}
\Deltav_{\xt}(\Btheta;\Btheta^*) 
= c_{\xt}(\Btheta,\Btheta^*)[c_{\xt}(\Btheta^*,\Btheta^*) + \diag(\sigma_n^2(\Btheta_1^*),\ldots,\sigma_n^2(\Btheta_b^*))]^{-1}c_{\xt}(\Btheta^*,\Btheta). \label{eq:gp_dvar}
\end{equation}
\end{lemma}
%

In the Lemma, $m_{\xtb}(\Btheta;\Btheta^*)$ is the GP mean function at iteration $t+b$ whose dependence on $\Btheta^*$ is shown explicitly. 
Importantly, the above Lemma shows how the GP variance decreases from $s^2_{\xt}(\Btheta)$ to $s^2_{\xtb}(\Btheta;\Btheta^*)$ when the extra $b$ evaluations at $\Btheta^*$ are included, and the reduction $\Deltav_{\xt}(\Btheta;\Btheta^*)$ is deterministic. 
We see, for example, that 
if $c_{\xt}(\Btheta_i^*,\Btheta_j^*)=0$ for all $i,j=1,\ldots,b, i\neq j$ which might hold approximately, e.g., if the evaluation points $\Btheta^*_i$ are located far from each other, then 
\begin{equation}
\Deltav_{\xt}(\Btheta;\Btheta^*) 
= \sum_{i=1}^b \Deltav_{\xt}(\Btheta;\Btheta^*_i)
= \sum_{i=1}^b \frac{c_{\xt}^2(\Btheta,\Btheta^*_i)}{s^2_{\xt}(\Btheta^*_i) + \sigma_n^2(\Btheta^*_i)}. \label{eq:gp_factorization}
\end{equation}
This shows that the reduction of GP variance at $\Btheta$, $\Deltav_{\xt}(\Btheta;\Btheta^*)$, factorises over the new evaluation points $\Btheta^*_i$ in $\Btheta^*$. Intuitively, if the test point $\Btheta$ is strongly correlated with some evaluation point $\Btheta^*_i$, including the evaluation at $\Btheta^*_i$ will result in a large reduction of variance at the test point. Furthermore, the larger the noise variance $\sigma_n^2(\Btheta^*_i)$ at the evaluation point $\Btheta^*_i$ is, the less the GP variance will decrease.

It clearly holds that $0 \leq \Deltav_{\xt}(\Btheta;\Btheta^*) \leq s^2_{\xt}(\Btheta)$. Items (i-ii) of the following Lemma summarise some additional properties of the variance reduction function in \eqref{eq:gp_dvar} and (iii-iv) show two further useful identities needed later. Item (i) shows the (rather obvious) result that the order of evaluation points in $\Btheta^*$ does not change $\Deltav_{\xt}(\Btheta;\Btheta^*)$ and in the following we often identify the $d\times b$ matrix $\Btheta^*$ with a multiset whose elements are the columns of $\Btheta^*$ although this leads to some abuse of notation. 
%
\begin{lemma} \label{lemma:2}
Let $\Btheta^*\in\reals^{d\times b}$ and let $\Btheta\in\reals^d$ be any test point. The function $\Btheta^*\mapsto\Deltav_{\xt}(\Btheta;\Btheta^*)$ in \eqref{eq:gp_dvar} for any fixed $\Btheta$ has the following properties. 
\vspace{-1.2em}
\begin{enumerate}[label=(\roman*)]
\itemsep-0.25em 
\item The function value is invariant to the permutation of the evaluation locations i.e.~the columns of $\Btheta^*$.
%
%
\item Let $\Btheta_A^* \subseteq \Btheta_B^*$. Then $\Deltav_{\xt}(\Btheta;\Btheta_A^*) \leq \Deltav_{\xt}(\Btheta;\Btheta_B^*)$, i.e., including new evaluations never increases variance.
\item Let $\Btheta^* = [\Btheta_1^*,\Btheta_2^*]$ so that $b=2$ and denote the predictive variance at $\Btheta_j^*$ by
$\bar{s}^2_{\xt}(\Btheta_j^*) \eqdef s^2_{\xt}(\Btheta_j^*)+\sigma^2_n(\Btheta_j^*)$ for $j\in\{1,2\}$. Then 
\begin{align}
\Deltav_{\xt}(\Btheta;\Btheta^*) 
&= \Deltav_{\xt}(\Btheta;\Btheta_1^*) + \Deltav_{\xt}(\Btheta;\Btheta_2^*) + r_{\xt}(\Btheta;\Btheta^*_1,\Btheta^*_2), \\
%
\begin{split}
r_{\xt}(\Btheta;\Btheta^*_1,\Btheta^*_2) &\eqdef \frac{c^2_{\xt}(\Btheta_1^*,\Btheta_2^*)c^2_{\xt}(\Btheta,\Btheta_1^*)\bar{s}^2_{\xt}(\Btheta_2^*) + c^2_{\xt}(\Btheta_1^*,\Btheta_2^*)c^2_{\xt}(\Btheta,\Btheta_2^*)\bar{s}^2_{\xt}(\Btheta_1^*)} 
%
{\bar{s}^4_{\xt}(\Btheta_1^*)\bar{s}^4_{\xt}(\Btheta_2^*) - \bar{s}^2_{\xt}(\Btheta_1^*)\bar{s}^2_{\xt}(\Btheta_2^*)c^2_{\xt}(\Btheta_1^*,\Btheta_2^*)}   \\
%
&\myquad- 2\frac{c_{\xt}(\Btheta,\Btheta_1^*)c_{\xt}(\Btheta,\Btheta_2^*)c_{\xt}(\Btheta_1^*,\Btheta_2^*)}
{\bar{s}^2_{\xt}(\Btheta_1^*)\bar{s}^2_{\xt}(\Btheta_2^*) - c^2_{\xt}(\Btheta_1^*,\Btheta_2^*)}. 
\end{split} 
\label{eq:for_submodularity}
\end{align}
\item Let $\Btheta^* = [\Btheta_A^*,\Btheta_b^*]$ where $\Btheta_A^*\in\reals^{d\times(b-1)}$ and $\Btheta_b^*\in\reals^{d}$, and denote $\bar{\BS}_A = c_{\xt}(\Btheta_A^*) + \diag(\sigma_n^2(\Btheta_{1}^*),\ldots,\sigma_n^2(\Btheta_{b-1}^*))$. Then
\begin{equation}\begin{aligned}
\Deltav_{\xt}(\Btheta;\Btheta^*) 
&= \Deltav_{\xt}(\Btheta;\Btheta_A^*) 
%
+\frac{(c_{\xt}(\Btheta,\Btheta_b^*) - c_{\xt}(\Btheta,\Btheta_A^*)\bar{\BS}_A^{-1}c_{\xt}(\Btheta_A^*,\Btheta_b^*))^2}
%
{s^2_{\xt}(\Btheta_b^*) + \sigma_n^2(\Btheta_b^*) - c_{\xt}(\Btheta_b^*,\Btheta_A^*)\bar{\BS}_A^{-1}c_{\xt}(\Btheta_A^*,\Btheta_b^*)}. \label{eq:batch_shortcut}
\end{aligned}\end{equation}
\end{enumerate}
\end{lemma}

Figure~\ref{fig:s2_gp} demonstrates how two new evaluations at $\theta_1^*$ and $\theta_2^*$ reduce the GP variance in two different one-dimensional examples. Specifically, Figure~\ref{fig:s2_gp}a illustrates the fact of Lemma \ref{lemma:2} (iii) that the interaction between the evaluation points, represented by the term $r_{\xt}(\theta;\theta_1^*,\theta_2^*)$, affects the reduction of the GP variance and its effect can be either positive or negative (the variance after two new evaluations, blue line, is either below or above the yellow line, which represents the reduced variance if the interaction is neglected). In Figure~\ref{fig:s2_gp}b the new evaluation locations are far apart and the factorisation of the variance reduction in \eqref{eq:gp_factorization} holds approximately.

\begin{figure*}[ht]
\centering
\includegraphics[width=0.85\textwidth]{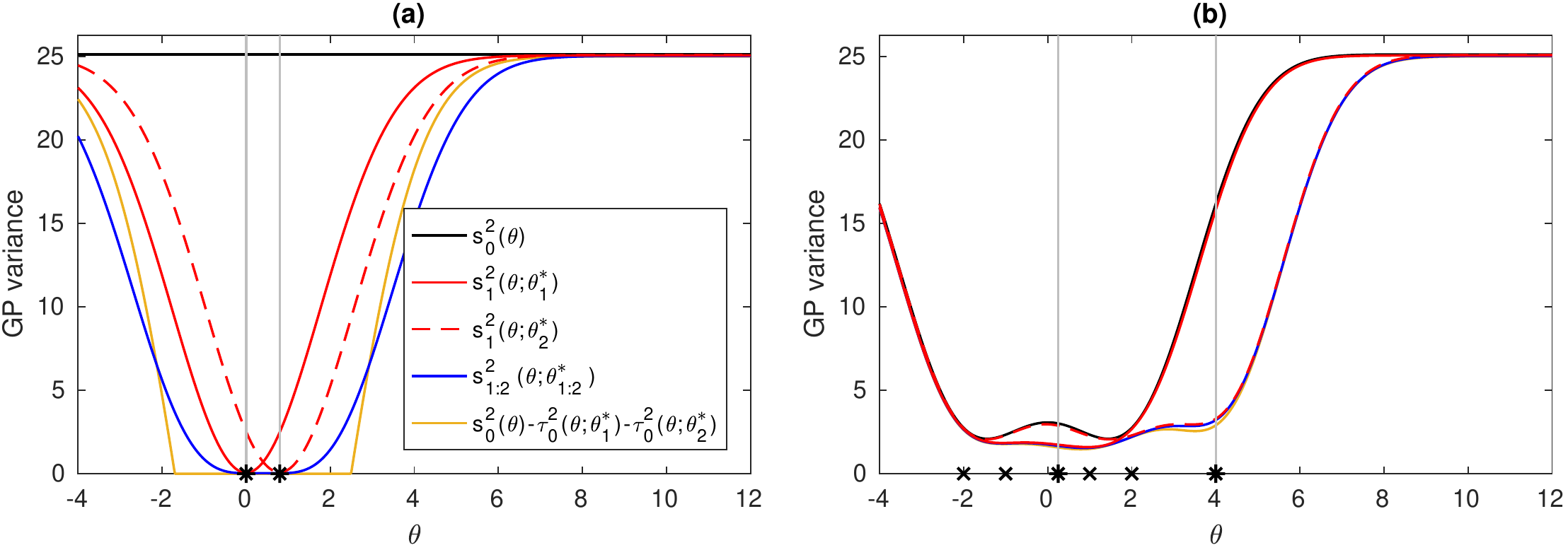}
\caption{The effect of two new evaluations (black stars) on the GP variance. Black line is the original variance, red lines (solid and dashed) show variance if only one of the evaluations is included. Blue line shows the variance after both evaluations, and yellow the variance if the interaction between the locations is neglected. (a) Noiseless observations at evaluation locations close to each other are obtained. (b) similar to (a), but showing four earlier evaluations (at black crosses) from which noisy observations were available, such that the GP variance is not exactly zero at these locations. }
\label{fig:s2_gp}
\end{figure*}

\subsection{Batch-sequential designs} \label{sec:seq_acq}

Given $D_{\xt} = \{(y_i,\Btheta_i)\}_{i=1}^t$, our goal is to select the next batch of $b$ evaluations $\Btheta^*$ in an optimal fashion. 
We take a Bayesian decision theoretic approach, where $\Btheta^*$ is selected to minimise the expected (or median) loss, where the loss measures uncertainty remaining in the unnormalised posterior $\tildepiapprox^f$ when the hypothetical observations $\By^*$ at locations $\Btheta^*$ are taken into account. In the following we develop two such techniques based on two different measures of uncertainty: variance and interquartile range (IQR). 
%
Design strategies which acknowledge the impact of the next batch, but neglect the whole remaining computational budget, are often called ``myopic''. It is possible to formulate a non-myopic design as a dynamic programming problem, but this is computationally demanding, see e.g.~\citet{Bect2012,gonzalez2015glasses}. Consequently, we focus on myopic designs which already produce highly sample-efficient and practical algorithms.

\subsubsection{Expected integrated variance (EIV)}\label{sec:eiv_derivation}

As our first measure of the uncertainty of the unnormalised posterior $\tildepiapprox^f$ for the selection of the next batch design $\Btheta^*$, 
we select the Bayes risk under the $L^2$ loss.
In this case, the Bayes risk is the integrated variance function 
\begin{equation}
\mathcal{L}^{\text{v}}(\probmeas_{D_{\xt}}^{\lik}) 
\eqdef \int_{\Theta} \Var_{\lik \cond D_{\xt}}(\tildepiapprox_{\lik}(\Btheta)) \ud\Btheta
= \int_{\Theta} \pi^2(\Btheta)\e^{2m_{\xt}(\Btheta) + s_{\xt}^2(\Btheta)}\left(\e^{s_{\xt}^2(\Btheta)} - 1\right) \ud\Btheta,
%
\end{equation}
whose integrand was obtained from \eqref{eq:log_mean_var}. 
This is similar to \citet{Sinsbeck2017,Jarvenpaa2018_acq} who, however, considered other GP surrogate models and only sequential designs.
%
We compute the expectation over the hypothetical noisy log-likelihoods $\By^*$ for any candidate design $\Btheta^*$, leading to the expected integrated variance design criterion (\eiv{}). The resulting optimal strategy is a special case of stepwise uncertainty reduction technique, see e.g.~\citet{Bect2012}. This criterion is evaluated efficiently without numerical simulations from the GP model, using the following result.
%
\begin{proposition} \label{prop:var}
With the assumptions of Lemma \ref{lemma:1}, the expected integrated variance design criterion $L_{\xt}^{\textnormal{v}}$ at any candidate design $\Btheta^*\in\Theta^b$ is
\begin{align}
L_{\xt}^{\textnormal{v}}(\Btheta^*)
\eqdef\mean_{\By^*\!\cond\!\Btheta^*, D_{\xt}}\mathcal{L}^{\textnormal{v}}(\probmeas_{D_{\xt}\cup D^*}^{\lik}) 
=\int_{\Theta}\!\pi^2(\Btheta) \e^{2m_{\xt}(\Btheta)\!+\! s^2_{\xt}(\Btheta)}\!\left(\!\e^{s^2_{\xt}(\Btheta)}\!-\! \e^{\Deltav_{\xt}(\Btheta;\Btheta^*)}\!\right)\!\ud\Btheta. \label{eq:eiv_prop}
\end{align}
\end{proposition}

\subsubsection{Integrated median interquartile range (IMIQR)}

A sequential design strategy based on \eiv{} worked well in the ABC scenario of \citet{Jarvenpaa2018_acq}, who modelled the discrepancy in \eqref{eq:abc_post} with a GP. 
In this article we instead model the log-likelihood with a GP as illustrated in Figure~\ref{fig:loglik_gp} and the goal is to minimise the uncertainty of the unnormalised posterior $\tildepiapprox^f$, which has a log-Normal distribution for a fixed $\Btheta$. 
However, the expectation and variance can be suboptimal estimates of the central tendency and uncertainty of a heavy-tailed distribution such as log-Normal. For example, Figure~\ref{fig:loglik_gp}b shows that the standard deviation (dashed blue line) grows very rapidly at the boundaries although at the same time the credible interval clearly indicates that the probability of the log-likelihood, and consequently the likelihood, of having a non-negligible value there is vanishingly small. The mean is similarly affected in a non-intuitive way by the heavy tails: 
It is fairly easy to see that 
\begin{equation}
%
\prob[\exp(\lik(\Btheta)) \geq \mean(\exp(\lik(\Btheta)))] = \prob(\lik(\Btheta) \geq m_t(\Btheta)+s^2_t(\Btheta)/2) = \Phi(-s_t(\Btheta)/2).
\end{equation}
This means that with a sufficiently large variance of the log-likelihood $s^2_t(\Btheta)$, the probability that the likelihood $\exp(\lik(\Btheta))$ is greater than its own mean becomes negligible.

The above analysis suggests (and empirical results in Section \ref{sec:experiments} further confirm) that mean-based point estimates and variance-based design strategies, such as the \eiv{} and those proposed by \citet{Gunter2014,Kandasamy2015,Sinsbeck2017,Jarvenpaa2018_acq,Acerbi2018}, are unsuitable when log-likelihood evaluations are noisy. 
%
A reasonable alternative for the $L^2$-loss used to derive the \eiv{} is to measure the uncertainty in the posterior using the $L^1$-loss, which is less affected by extreme values. As shown in Section \ref{sec:gp_model}, the $L^1$-loss leads to the marginal median estimate for the posterior, $\pi_{\xt}^{\text{med}}$. 
%
While the $L^1$-loss produces a robust median estimator that we adopt, \eqref{eq:loneloss} shows that the Bayes risk with $L^1$ loss scales as $\exp(s^2_{\xt}(\Btheta)/2)$ since $\Phi(s_{\xt}(\Btheta))\approx 1$ for large $s_{\xt}(\Btheta)$, such that also this measure for overall uncertainty of $\tildepiapprox^f$ is affected by the heavy tails of the log-Normal distribution. 

We propose a new, robust design criterion for selecting the next design.
In place of the variance in \eiv{}, we use a robust measure of uncertainty, the interquartile range $\text{\iqr{}}(\Btheta)=q^{3/4}(\Btheta)-q^{1/4}(\Btheta)$.
%
The integrated \iqr{} loss measuring the uncertainty of the posterior pdf is defined as
%
%
\begin{equation}
\mathcal{L}^{\text{\iqr}}(\probmeas_{D_{\xt}}^{\lik}) 
\eqdef \int_{\Theta} \text{\iqr}_{\lik \cond D_{\xt}}(\tildepiapprox_{\lik}(\Btheta)) \ud\Btheta
%
%
= 2\int_{\Theta} \pi(\Btheta) \e^{m_{\xt}(\Btheta)} \sinh(us_{\xt}(\Btheta)) \ud\Btheta, \label{eq:iiqrloss}
\end{equation}
where $u\eqdef\Phi^{-1}(p_u)$ and $\sinh(z) = (\exp(z)-\exp(-z))/2$ for $z\in\reals$ is the hyperbolic sine, which emerges after using \eqref{eq:log_med_q}. While we use $p_u=0.75$, other quantiles $p_u\in(0.5,1)$ are also possible. 
A theoretical downside of the \iqr{} loss is that it does not formally coincide with the Bayes risk for the $L^1$ or $L^2$ loss, which correspond to the optimal point estimators of the unnormalised posterior (see Section \ref{sec:estimators}).

We also use the median in place of the mean to measure the effect of the next design $\Btheta^*$ to the loss function. That is, we use median loss decision theory (see \citet{Yu2011}), 
and define the median integrated \iqr{} loss function as 
\begin{equation}\begin{aligned}
L^{\text{\iqr}}_{\xt}(\Btheta^*)
& \eqdef \med_{\By^* \cond  \Btheta^*, D_{\xt}}\mathcal{L}^{\text{\iqr}}(\probmeas_{D_{\xt}\cup D^*}^{\lik}). 
%
\label{eq:true_med_loss}
\end{aligned}\end{equation}
The median integrated \iqr{} loss in \eqref{eq:true_med_loss} is intractable but it can be approximated by the integrated median \iqr{} loss (\miiqr{}). This approximation\footnote{Instead of approximation, which may be inaccurate when the GP variance function is large, the integrated median criterion in \eqref{eq:miiqr_prop} can be seen as an alternative decision-theoretic formulation with infinitely many dependent variables of interest (one for each $\Btheta\in\Theta$) and where the median outcomes after considering the effect of the design $\Btheta^*$ are all computed separately for each $\Btheta\in\Theta$, and the corresponding losses are combined through averaging. The median integrated loss in \eqref{eq:true_med_loss} instead has a single combined loss function. } follows by replacing the predictive distribution of $\By^*$ with a point mass, i.e.,  $\pi(\By^*\cond \Btheta^*,D_{\xt}) \approx \delta(m_{\xt}(\Btheta^*)-\By^*)$. This approximation resembles the so-called kriging believer heuristic in \citet{Ginsbourger2010}. The next result gives a useful formula to calculate \miiqr{}. 
%
\begin{proposition} \label{prop:iqr}
With the assumptions of Lemma \ref{lemma:1}, 
the integrated median \iqr{} loss, denoted as $\tilde{L}^{\text{\iqr}}_{\xt}$, at any candidate design $\Btheta^*\in\Theta^b$ is
\begin{equation}\begin{aligned}
\tilde{L}^{\text{\iqr}}_{\xt}(\Btheta^*)
&\!\eqdef\!\int_{\Theta}\!\textnormal{\med}_{\By^*\!\cond\!\Btheta^*,D_{\xt}}\text{\iqr}_{\lik \!\cond\! D_{\xt}\cup D^*}(\tildepiapprox_{\lik}(\Btheta))\!\ud\Btheta 
%
\!=\!2\!\int_{\Theta}\!\pi(\Btheta) \e^{m_{\xt}(\Btheta)} \sinh(us_{\xtb}(\Btheta;\Btheta^*))\!\ud\Btheta.  \label{eq:miiqr_prop}
\end{aligned}\end{equation}
\end{proposition}
%
The integrand of \eqref{eq:miiqr_prop} is recognised as a product of the marginal median estimate of the posterior in \eqref{eq:log_med_q} and the function $\sinh(us_{\xtb}(\Btheta;\Btheta^*))$.
Hence, to minimise \miiqr{}, the simulation locations $\Btheta^*$ need to be chosen as a compromise between regions where the current posterior estimate is non-negligible and where the GP variance $s^2_{\xtb}(\Btheta;\Btheta^*)$ decreases efficiently when the simulations are run at $\Btheta^*$.
Similar interpretation holds also for the \eiv{} function in \eqref{eq:eiv_prop}. However, \eiv{} assigns significantly more weight to areas with high GP variance than \miiqr{}. 

\subsection{Joint and greedy optimisation for batch-sequential designs} \label{sec:greedy_acq}

We can now evaluate \eiv{} and \miiqr{} design criteria for any candidate design $\Btheta^*$ and choose $\Btheta^*$ as the minimiser, i.e.,
\begin{equation}
\Btheta^* = \arg\min_{\Btheta\in\Theta^b} L_{\xt}(\Btheta), \label{eq:joint_batch}
\end{equation}
where $L_{\xt}$ is either the \eiv{} in \eqref{eq:eiv_prop} or \miiqr{} in \eqref{eq:miiqr_prop}. 
The objective function is typically smooth but multimodal so global optimisation is needed. 
%
We call \eqref{eq:joint_batch} as ``joint'' optimisation which does not scale to high dimensional parameter spaces or to large batch sizes. Even if computing the design criterion is cheap as compared to the run times of typical simulation models, solving the $db$-dimensional global optimisation problem is often impractical as discussed in \citet{Wilson2018}. 
Hence, we consider greedy optimisation as also used in batch BO \citep{Ginsbourger2010,Snoek2012,Wilson2018}. The greedy optimisation procedure for both \eiv{} and \miiqr{} works as follows: the first point $\Btheta_1^*$ is chosen as in the sequential case i.e.~by solving \eqref{eq:joint_batch} with $b=1$. The rest of the points $\Btheta^*_{2:b}$ are obtained by iteratively solving
\begin{equation}
\Btheta_r^* = \arg\min_{\Btheta\in\Theta}L_{\xt}([\Btheta_{1:r-1}^*,\Btheta]), \quad r = 2,3,\ldots,b. \label{eq:greedy_optim}
\end{equation} 
This greedy approach simplifies the difficult $db$-dimensional optimisation into $b$ separate $d$-dimensional problems, and makes it scalable as a function of $b$.

In general, the design found by the greedy optimisation does not equal the minimiser of the joint criterion. It follows from Lemma \ref{lemma:2} (i) that both \eiv{} and \miiqr{} are invariant to the order of evaluation locations in $\Btheta^*$ but this does not hold for the greedy procedure. Bounds for the performance of greedy maximisation of a set function have been studied in literature, see e.g.~\citet{Nemhauser1978,Krause2008,Bach2013}.  
For example, if the design criterion (when defined equivalently using a utility so that \eqref{eq:greedy_optim} becomes a maximisation problem) is submodular and non-decreasing in batch size $b$, then the worst-case outcome of greedy optimisation is at least $1-1/\e\approx 0.63$ of the corresponding optimal joint value. 
A utility function corresponding to \miiqr{} defined below is not submodular but an approximation of it is weakly submodular (see e.g.~\citet{Krause2008,Krause2010}). We use this fact to derive a weaker but still useful bound. 

We here consider an approximation $\tilde{L}^{\text{\iqr},a}_{\xt}(\Btheta^*)$ of $\tilde{L}^{\text{\iqr}}_{\xt}(\Btheta^*)$ so that 
\begin{align}
    \tilde{L}^{\text{\iqr}}_{\xt}(\Btheta^*) 
    %
    %
    \approx \tilde{L}^{\text{\iqr},a}_{\xt}(\Btheta^*) 
    \eqdef 2u^2\int_{\Theta} \pi(\Btheta) \e^{m_{\xt}(\Btheta)} s^2_{\xtb}(\Btheta;\Btheta^*) \ud\Btheta, \label{eq:U_miiqr_a}
\end{align}
which follows from the observation that $\sinh(us_{\xtb}(\Btheta;\Btheta^*)) \approx u^2 s^2_{\xtb}(\Btheta;\Btheta^*)$ and where we had $u=\Phi^{-1}(p_u)$. 
The approximation in \eqref{eq:U_miiqr_a} is reasonable when $s_{\xtb}(\Btheta;\Btheta^*)\in[0,3/u]$ in the region where $\pi(\Btheta)\e^{m_{\xt}(\Btheta)}$ is non-negligible. For simplicity, we consider a discretised setting where the optimisation is done over a finite set $\tilde{\Theta}\subset\Theta$ and define (approximate) \miiqr{} utility function as 
\begin{align}
    \tilde{U}^{\text{\iqr},a}_{\xt}(\Btheta^*) \eqdef \tilde{L}^{\text{\iqr},a}_{\xt}(\emptyset) - \tilde{L}^{\text{\iqr},a}_{\xt}(\Btheta^*) \label{eq:miiqr_utility}
\end{align}
for $\Btheta^* \in 2^{\tilde{\Theta}}$. Clearly, maximising $\tilde{U}^{\text{\iqr},a}_{\xt}(\Btheta^*)$ is equivalent to minimising $\tilde{L}^{\text{\iqr},a}_{\xt}(\Btheta^*)$. 
The following theorem gives a bound for the greedy optimisation of the (approximate) \miiqr{} utility function. The proof can be found in \appe{} \ref{subsec_greedyboundproof}.
%
\begin{theorem} \label{thm:miiqr_submod}
Consider the set function $\tilde{U}^{\text{\iqr},a}_{\xt}: 2^{\tilde{\Theta}} \rightarrow \reals_{+}$ in \eqref{eq:miiqr_utility}. Let $\Btheta_O$ be a (joint) optimal solution for maximising $\tilde{U}^{\text{\iqr,a}}_{\xt}(\Btheta)$ over $\Btheta\subset\tilde{\Theta}, |\Btheta| \leq b$. The greedy algorithm for this maximisation problem outputs a set $\Btheta_{G}\subset\tilde{\Theta}$ satisfying
\begin{align}
    &\tilde{U}^{\text{\iqr},a}_{\xt}(\Btheta_{G}) 
    %
    %
    \geq (1-1/e) \tilde{U}^{\text{\iqr},a}_{\xt}(\Btheta_O) - b^2\epsilon_{\xt}, \quad \text{where}\\
    &\epsilon_{\xt} \eqdef \max\{0,2u^2\epsilon'_{\xt}\}, 
    \quad \epsilon'_{\xt} \eqdef \max_{\substack{\Btheta_A\subset\tilde{\Theta}, |\Btheta_A|=i \leq 2b \\ \Btheta_j,\Btheta_k\in\tilde{\Theta}}} \int_{\Theta} \pi(\Btheta)\e^{m_{\xt}(\Btheta)}r_{1:t+i}(\Btheta;\Btheta_j,\Btheta_k)\ud\Btheta.
\end{align}
\end{theorem}
%
%
Computing $\epsilon_{\xt}$ explicitly is difficult but we expect that often $\epsilon_{\xt} \ll \tilde{U}^{\text{\iqr},a}_{\xt}(\Btheta_O)$ since $r_{1:t+i}(\Btheta;\Btheta_j,\Btheta_k)$ given by \eqref{eq:for_submodularity} tends to be small. 
However, $\epsilon_{\xt}$ may not always be small, the term $b^2\epsilon_{\xt}$ scales quadratically for batch size $b$, and the bound holds only approximately for \miiqr{}. This bound still suggests that, at least in some iterations of the algorithm, greedy \miiqr{} produces near-optimal batch locations. On the other hand, an approximation similar to \eqref{eq:U_miiqr_a} for \eiv{} would be reasonable in a very limited number of situations, and experiments in \appe{} \ref{subsec:expextratwod} suggest that greedy \eiv{} scales worse as a function of $b$ compared to the corresponding greedy \miiqr{} strategy. 
%
Finally, we note that even when the bound is weak, new design points cannot increase the value of \eiv{} or \miiqr{} loss function as shown in \appe{} \ref{sec:monotonicity}.
Hence, the batch strategies cannot be worse than the corresponding sequential designs and, in practice, they are highly useful as is seen empirically in Section \ref{sec:experiments}.

\subsection{Implementation details} \label{sec:num_details}

Using the GP surrogate model and the analysis from the previous sections, we show the resulting inference method as Algorithm \ref{alg:gp_sl_alg}. Some key implementation details are discussed below and further details (e.g.~on handling GP hyperparameters, MCMC methods used, and optimisation of the design criteria) are given in \appe{} \ref{app:add_details}. 
The algorithm is shown for the SL case using the \miiqr{} strategy, but it works similarly for \eiv{}, heuristic designs developed in the next section and other log-likelihood estimators besides SL. The potentially expensive simulations on the lines 2-5 and 18-21 can be done in parallel. In the SL case, the simulations can be parallelised in terms of both the number of repeated simulations $N$ and batch size $b$. 
\begin{algorithm}[htb]
\caption{GP-based SL inference using \miiqr{} with synchronous batch design \label{alg:gp_sl_alg}} 
 \begin{algorithmic}[1]
 \Require Prior density $\pi(\Btheta)$, simulation model $\pi(\cdot \cond \Btheta)$, GP prior $\probmeas^{\lik}$, number of repeated samples $N$, summary function $S$, batch size $b$, initial batch size $b_0$, max.~iterations $\iter_{\text{max}}$, number of IS samples $s_{\text{IS}}$, number of MCMC samples $s_{\text{MC}}$ 
 \State Sample $\Btheta_{1:b_0} \simiid \pi(\cdot)$ (or use some other space-filling initial design)
 %
 \For{$r=1:b_0$}
 \State Simulate $\Bx_{r}^{(1:N)} \simiid \pi(\cdot \cond \Btheta_{r})$ and compute $S_r^{(1:N)} = S(\Bx_{r}^{(1:N)})$ 
 \hspace{0em}\raisebox{-.5\baselineskip}[0pt][0pt]{$\left.\rule{0pt}{2.3\baselineskip}\right\}\ \mbox{in parallel}$}
 \State Compute $y_r$ from $\{S_r^{(j)}\}_{j=1}^N$
 \EndFor
 \State Set initial training data $D_{b_0} \leftarrow \{(y_{r}, \Btheta_{r})\}_{r=1}^{b_0}$
 %
 \For{$\iter=1:\iter_{\text{max}}$} 
  %
  \State Use MAP estimation to obtain GP hyperparameters $\Bphi$ using $D_{b_0+(\iter-1)b}$ 
  %
  \State Sample $\Btheta^{(j)}\sim\pi_q$ using MCMC and compute $\omega^{(j)}$ in Eq.~\ref{eq:is_approx} for $j = 1,\ldots,s_{\text{IS}}$ 
  %
  \If{\texttt{joint\_optim}}
  \State Obtain $\Btheta_{1:b}^{(\iter)*}$ by solving Eq.~\ref{eq:joint_batch} using Eq.~\ref{eq:miiqr_prop} and \ref{eq:is_approx}
  \ElsIf{\texttt{greedy\_optim}}
  	\State Obtain $\Btheta_{1}^{(\iter)*}$ by minimising Eq.~\ref{eq:joint_batch} using Eq.~\ref{eq:miiqr_prop} and \ref{eq:is_approx}
  \For{$r=2:b$}
  	\State Obtain $\Btheta_{r}^{(\iter)*}$ by solving Eq.~\ref{eq:greedy_optim} using Eq.~\ref{eq:miiqr_prop} and \ref{eq:is_approx}
  \EndFor
  \EndIf
  \For{$r=1:b$}
 	\State Simulate~$\Bx_{r}^{(\iter,1:N)}\!\simiid\!\pi(\cdot \cond \Btheta_{r}^{(\iter)*})$,~compute~$S_r^{(\iter,1:N)}\!= \!S(\Bx_{r}^{(\iter,1:N)})$ 
 	\hspace{0em}\raisebox{-.5\baselineskip}[0pt][0pt]{$\left.\rule{0pt}{2.3\baselineskip}\right\}\ \mbox{in parallel}$}
 \State Compute $y_r^{(\iter)*}$ using $\{S_r^{(\iter,j)}\}_{j=1}^N$
 \EndFor
  \State Update training data $D_{b_0+ib} \leftarrow D_{b_0+(i-1)b} \cup \{(y_{r}^{(\iter)*}, \Btheta_{r}^{(\iter)*})\}_{r=1}^b$
 \EndFor
 \State Use MAP estimation to obtain GP hyperparameters $\Bphi$ using $D_{b_0 + \iter_{\text{max}}b}$
 \State Sample $\Bvartheta^{(1:s_{\text{MC}})}$ from the marginal median estimate in Eq.~\ref{eq:point_estim_norm} using MCMC
 \State \Return Samples $\Bvartheta^{(1:s_{\text{MC}})}$ from the approximate SL posterior
 \end{algorithmic}
\end{algorithm}

Evaluation of \eiv{} and \miiqr{} requires numerical integration over $\Theta\subset\reals^d$. Similar computational challenges emerge also in the state-of-the-art BO methods such as \citet{Hennig2012,HernandezLobato2014,Wu2016} and in \citet{Chevalier2014}. If $d\leq2$ we discretise the parameter space $\Theta$ and approximate the integral in the resulting grid. In higher dimensions, we use self-normalised importance sampling (IS) as in \citet{Chevalier2014,Jarvenpaa2018_acq}. Specifically, we draw samples from the importance distribution $\Btheta^{(j)}\sim\pi_q(\Btheta)$ and use these as integration points to approximate 
\begin{align}
\int_{\Theta} I_{\xt}(\Btheta;\Btheta^*) \ud\Btheta 
%
%
\approx \sum_{j=1}^{s_{\text{IS}}} \omega^{(j)} I_{\xt}(\Btheta^{(j)};\Btheta^*), \quad \omega^{(j)} = \frac{1/\pi_q(\Btheta^{(j)})}{\sum_{k=1}^s 1/\pi_q(\Btheta^{(k)})}, 
\label{eq:is_approx}
\end{align}
where the integrand of either \eqref{eq:eiv_prop} or \ref{eq:miiqr_prop} is denoted by $I_{\xt}(\Btheta;\Btheta^*)$. As the proposal $\pi_q$ we use the current loss (see \eqref{eq:log_mean_var} and the integrand of \eqref{eq:loneloss}), which is a function $\Theta\rightarrow\reals_{+}$, and can be interpreted as an unnormalised pdf. This is a natural choice because the current loss has a similar shape as the expected/median loss as a function of $\Btheta$. 
We use the same proposal in the greedy optimisation in \eqref{eq:greedy_optim} although it would be also possible to adapt the proposal $\pi_q$ according to the pending points $\Btheta_{1:r-1}^*$ when optimising with respect to the $r$th point $\Btheta^*_r$. 

We have assumed that the noise function $\sigma_n^2$ in \eqref{eq:noisemodel} is known. In practice, this is a valid assumption only in the noiseless case where $\sigma_n^2(\Btheta) = 0$. 
As our focus is on the noisy setting, we need to estimate $\sigma_n^2$. Sometimes $\sigma_n^2$ can be assumed to be an unknown constant to be determined together with the GP hyperparameters $\Bphi$ using MAP estimation. 
However, we observed that $\sigma_n^2$ often depends on the magnitude of the log-likelihood (see Figure~\ref{fig:loglik_gp}), making the assumption of homoscedastic noise questionable. 
Similarly to \citet{Wilkinson2014}, we estimate $\sigma_n^2$ using the bootstrap. Specifically, with each new training data point $\Btheta_i$, we resample with replacement $N$ summary vectors from the original population $\{S_i^{(j)}\}_{j=1}^N$ for $2000$ times. We then compute the empirical variance of the resulting log-SL values and use it as a plug-in estimator for $\sigma_n^2(\Btheta_i)$. 

For \eiv{}, \miiqr{} and greedy batch versions of \maxv{} and \maxiqr{} (see Section \ref{sec:other_acq}), $\sigma_n^2$ needs to be also known at candidate design points $\Btheta^*$. Bootstrap cannot be used because the simulated summaries are only available for training data. We take a pragmatic approach and set $\sigma_n=10^{-2}$ at the candidate design points as if the future evaluations were almost exact. 
This simplification effectively reduces the occurrence of (potentially redundant) simulations at nearby points to encourage exploration.
Alternatively, one could use another GP to model the bootstrapped variances or their logarithms and use the GP mean function as a point estimate for the function $\sigma_n^2$ 
as in \citet{Ankenman2010}. 

\subsection{Alternative heuristic designs strategies} \label{sec:other_acq}

Here we present some heuristic alternative design strategies. These are empirically compared to the more principled \eiv{} and \miiqr{} strategies in Section \ref{sec:experiments}. We first focus on sequential designs where $b=1$. 

\textbf{MAXIQR}: A natural and simple approach is to evaluate where the current variance, IQR or some other suitable (local) measure of uncertainty is maximised. Such strategies are in some contexts called ``uncertainty sampling''. The advantage over \eiv{} and \miiqr{} is cheaper computation because the effect of the candidate design point to the whole posterior is not acknowledged. 
%
Using IQR produces the design strategy
\begin{equation}
\Btheta^* = \arg\max_{\Btheta\in\Theta} \pi(\Btheta)\e^{m_{\xt}(\Btheta)} \sinh(us_{\xt}(\Btheta)), \label{eq:miqr}
\end{equation}
which we abbreviate as \maxiqr{} because it evaluates at the maximiser of \iqr{}. 
%
Taking the logarithm of \eqref{eq:miqr}, the \maxiqr{} strategy can be equivalently written as 
\begin{equation}
\Btheta^* = \arg\max_{\Btheta\in\Theta} \left(\log\pi(\Btheta) + {m_{\xt}(\Btheta)} + us_{\xt}(\Btheta) + \log(1-\e^{-2us_{\xt}(\Btheta)})\right), \label{eq:log_miqr}
\end{equation}
which shows a tradeoff between evaluating where the log-posterior is presumed to be large (the first two terms in \eqref{eq:log_miqr}) and unexplored regions where the GP variance is large (the last two terms). 
This formula also shows an interesting connection to the upper confidence bound (UCB) criterion commonly used in BO, see e.g.~\citet{Srinivas2010,Shahriari2015}. The UCB acquisition function is $\text{UCB}(\Btheta) = m_{\xt}(\Btheta) + \beta_t s_{\xt}(\Btheta)$, where $\beta_t$ is a tradeoff parameter, here automatically chosen to be $\beta_t = \Phi^{-1}(p_u)$. Compared to the standard UCB, there is, however, an extra term in \eqref{eq:log_miqr} which further penalises regions having small variance $s_{\xt}^2$. If the variance $s_{\xt}^2$ is large everywhere and/or if $p_u$ is an extreme quantile, then the last term in \eqref{eq:log_miqr} is approximately zero and the \maxiqr{} design criterion approximately equals UCB.

\textbf{MAXV}: When the variance is used instead of IQR, we obtain a strategy
\begin{equation}
\Btheta^* = \arg\max_{\Btheta\in\Theta} \pi^2(\Btheta)\e^{2m_{\xt}(\Btheta) + s_{\xt}^2(\Btheta)}\left(\e^{s_{\xt}^2(\Btheta)} - 1\right). \label{eq:maxv}
\end{equation}
This strategy is abbreviated as \maxv{} which,
in fact, is used by \citet{Gunter2014,Kandasamy2015} in the noiseless case, and it is called ``exponentiated variance'' by \citet{Kandasamy2015}. 
Taking logarithm of \eqref{eq:maxv} shows that this design also features a tradeoff between large posterior and large variance, similarly to \maxiqr{}.

Since these two strategies are not derived from Bayesian decision theory, it is not immediately clear how one should parallelise these inherently sequential strategies. However, it seems reasonable to use the fact the $s^2_{\xt}(\Btheta)$ is always reduced near the pending evaluation locations. Motivated by this and related BO techniques in \citet{Ginsbourger2010,Snoek2012,Desautels2014}, we compute the median value of the design criterion with respect to the posterior predictive distribution of the pending simulations. The next locations are chosen iteratively such that, for \maxiqr{}, the first point $\Btheta_1^*$ in the batch is chosen using \eqref{eq:miqr} and the rest $\Btheta_{2:b}$ by iteratively solving 
\begin{equation}
\Btheta_r^* = \arg\max_{\Btheta\in\Theta} \pi(\Btheta)\e^{m_{\xt}(\Btheta)} \sinh(us_{t+r-1}(\Btheta;\Btheta_{1:r-1}^*)), \quad r=2,3,\ldots,b. \label{eq:miqr_batch}
\end{equation}
%
\maxv{} is parallelised similarly but using the expected value instead of the median.

Finally, we provide some intuition to \eqref{eq:miqr_batch} and show a connection to the local penalisation method used to parallelise sequential BO designs by \citet{Gonzalez2016}. Suppose we are selecting the $r$th point of a batch where $2\leq r \leq b$. 
Comparison of \eqref{eq:miqr} and \eqref{eq:miqr_batch} shows that \eqref{eq:miqr_batch} equals the original design criterion in \eqref{eq:miqr} multiplied by a weight function $\omega(\Btheta;\Btheta^*_{1:r-1}) \eqdef \sinh(us_{t+r-1}(\Btheta;\Btheta^*_{1:r-1})) / \sinh(us_{\xt}(\Btheta))$. It is easy to see that $\omega(\Btheta;\Btheta^*_{1:r-1}) \in [0,1]$. This shows that when we take the median over the log-likelihood evaluation at the pending points $\Btheta^*_{1:r-1}$, we are implicitly making the original acquisition function smaller around the pending points and, consequently, penalising additional evaluations there. This resembles the heuristic method by \citet{Gonzalez2016}, who proposed to multiply the non-negative acquisition function, such as the objective of \eqref{eq:miqr}, with $\prod_{1\leq j<r} \varphi(\Btheta;\Btheta_j^*)$, where $\varphi(\Btheta;\Btheta_j^*)$ are local penalising functions around the pending evaluation locations $\Btheta_j^*$, when selecting the $r$th point $\Btheta_r^*$ of the current batch. 
%
%
However, one difference between these approaches is that our weight function $\omega$ takes the interactions between the pending points into account and it cannot be factorised as $\omega(\Btheta,\Btheta_{1:r-1}^*) = \prod_{1\leq j<r}\varphi(\Btheta;\Btheta_j^*)$. Also, our weight function is not a tuning parameter but follows automatically from our analysis. 
 

\section{Experiments} \label{sec:experiments}

We empirically investigate the performance of the proposed algorithm with different design strategies developed in Section \ref{sec:acq}. 
We compare the sequential, batch, and greedy batch strategies based on \eiv{} and \miiqr{} to sequential and greedy versions of \maxv{} (which is essentially the same as the BAPE method by \citet{Kandasamy2015}) and \maxiqr{}. As a simple baseline we also sample design points from the prior (always uniform) and this method is abbreviated as RAND. 

We report the results as figures whose y-axis shows
the accuracy between the estimated and the ground truth posterior using total variation distance (TV). TV between pdfs $\pi_1$ and $\pi_2$ is defined as
$\text{TV}(\pi_1,\pi_2) = 1/2\int_{\Theta}|\pi_1(\Btheta) - \pi_2(\Btheta)|\ud\Btheta$ and is computed using numerical integration in 2D. In higher dimensional cases we compute the average TV between the marginal posterior densities using MCMC samples. The marginal median estimator in \eqref{eq:point_estim_norm} is used to obtain the point estimate for the posterior pdf. The x-axis represents the iteration $\iter$ of the Algorithm \ref{alg:gp_sl_alg} (unless explicitly stated otherwise) which serves as a proxy to the total wall-time when the noisy likelihood evaluations are assumed to dominate the total computational cost. 
We use a fixed simulation budget so that the batch-sequential methods terminate earlier than the sequential ones because they spend the evaluation budget $b$ times faster due to the parallel computation. 

We consider two sets of experiments: toy models where noisy log-likelihood evaluations are directly evaluated (Section \ref{sec:synth_experiments}) and real-world simulator-based statistical models where SL is used to obtain noisy log-likelihood evaluations using $N$ repeated simulations at each proposed parameter (Section \ref{sec:real_example}). Although in the SL case it is often possible to adjust $N$ adaptively, we use $N=100$ (unless explicitly stated otherwise) for simplicity. 
In the \appe{} \ref{subsec:N_parallellisation} we show that our batch methods are beneficial for SL even though in principle it would be possible to directly parallelize the $N$ simulations themselves. 
For example, when $1,000$ computer cores are available for the simulations (e.g.~in a high performance computing cluster), 
it is beneficial to use the batch strategies (e.g.~$N=100$ and $b=10$) instead parallelizing the evaluations at a single location ($N=1000$ and $b=1$). 
More elaborate analysis of resource allocation is left for future work. 
%
A MATLAB implementation of our algorithms is available at \url{https://github.com/mjarvenpaa/parallel-GP-SL}.

\subsection{Noisy toy model likelihoods} \label{sec:synth_experiments}

We first define three 2D densities with different characteristics: a simple Gaussian density called 'Simple', a banana-shaped density 'Banana' and a bimodal density 'Bimodal'. We then construct three 6D densities so that their 2D blocks are independent and have the corresponding 2D densities as their 2D marginals. Detailed specification and illustrations can be found in \appe{} \ref{subsec:app_2d}. We use the same names for the 6D densities as for the corresponding 2D ones (except that 'Bimodal' is called 'Multimodal' in 6D because it has $2^3=8$ modes). The independence assumption is not taken into account in the GP model to make the inference problem more challenging. 
For simplicity, $\sigma_n(\Btheta)$ is assumed constant i.e.~it does not depend on the magnitude of the log-likelihood and its value is obtained using MAP estimation together with other GP hyperparameters $\Bphi$ at each iteration $\iter$ in Algorithm \ref{alg:gp_sl_alg}.
%
As an initial design in 6D we generate $b_0=20$ parameters ('Simple') or $b_0=50$ ('Banana' and 'Multimodal') from uniform priors. In the 2D case we always use $b_0=10$. We use a fixed total budget of $t=620$ noisy log-likelihood evaluations ('Simple') or $t=650$ ('Banana' and 'Multimodal') for both sequential and batch methods in 6D.

The results with different sequential and greedy batch-sequential strategies in 6D case with batch size $b=5$ are shown in Figure~\ref{fig:synth_res6d}. 
Good posterior approximations for the Simple example are obtained earlier than for the two other models. This is a consequence of the quadratic terms in the GP prior mean function and the exact Gaussian shape of the posterior. However, more complicated posteriors are also estimated accurately although more iterations are needed to obtain reasonable approximations. The \miiqr{} method works clearly the best outperforming \eiv{} and the heuristic \maxv{} and \maxiqr{} methods which either need more iterations to obtain good approximation or fail completely to reach good results. The uniform design RAND works adequately in the Simple and Banana models but often produces poor estimates for the Multimodal case. Unsurprisingly, its performance is also poor in the real-world scenarios in Section \ref{sec:real_example}.

\begin{figure*}[!htb]
\centering
\includegraphics[width=0.99\textwidth]{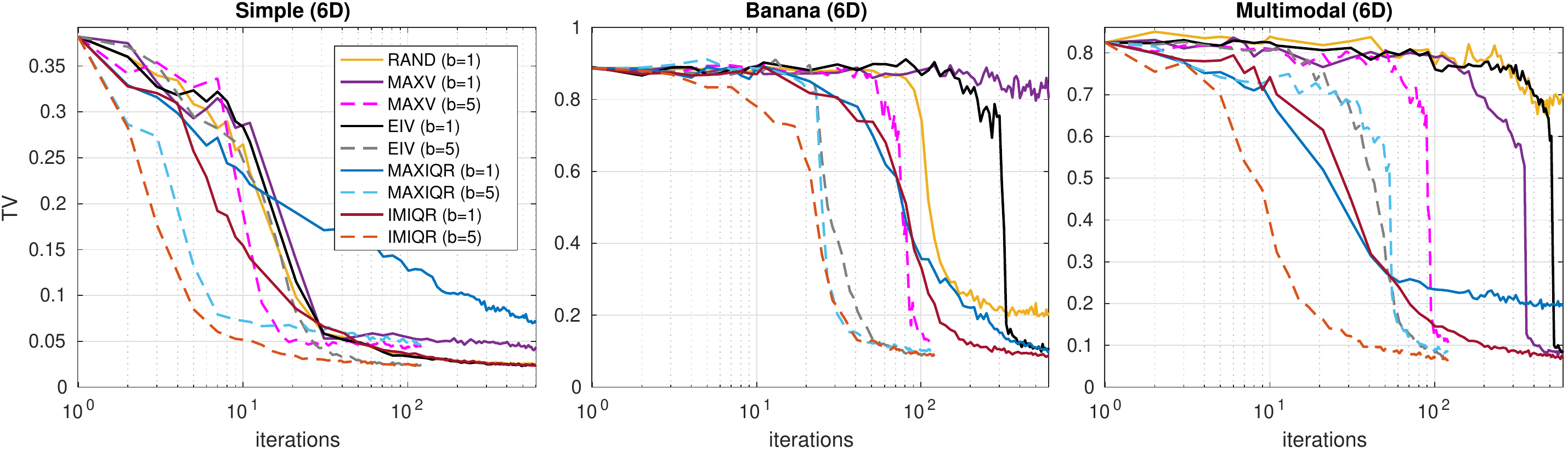}
\caption{Results for the 6D toy densities. The lines show the median TV over $50$ repeated simulations. Note that x-axis is on log-scale and the maximum number of iterations for the sequential methods is $i=600$ and for batch methods $i=120$. 
}
\label{fig:synth_res6d}
\end{figure*}

The batch-sequential strategies improve the convergence speed as compared to the corresponding sequential strategies in all cases of Figure~\ref{fig:synth_res6d}. In particular, the greedy batch versions of \maxv{} and \maxiqr{} even outperform the corresponding sequential methods. The greedy batch strategy in these cases encourages exploration as compared to the corresponding sequential strategy and this effect counterbalances the exploitative nature of \maxv{} and \maxiqr{}. 
In \appe{} \ref{subsec:expextratwod} we compare the joint and greedy batch strategies in 2D case where joint maximisation is still feasible. Their difference is found small for \miiqr{} and small or moderate for \eiv{} suggesting that the greedy strategies are in practice nearly optimal. 

Figure~\ref{fig:acq_points_demo_part2} and further examples in the \appe{} \ref{subsec:expextratwod} illustrate the design points and estimated posteriors for various design strategies in 2D case. 
An important observation is that \maxv{} and \miiqr{} are exploitative i.e.~they produce points near the mode of the posterior where the local measure of uncertainty they use tend to be highest. Also, in general, the sequential and batch methods produce similar designs. However, greedy \maxiqr{} generates more points on the boundary than the corresponding sequential strategy and the joint \miiqr{} produces slightly more diverse design points as the sequential and greedy batch \miiqr{}. In all cases, \miiqr{} avoids redundant evaluations on the boundaries.

\begin{figure*}[ht!]
\centering
\includegraphics[width=0.90\textwidth]{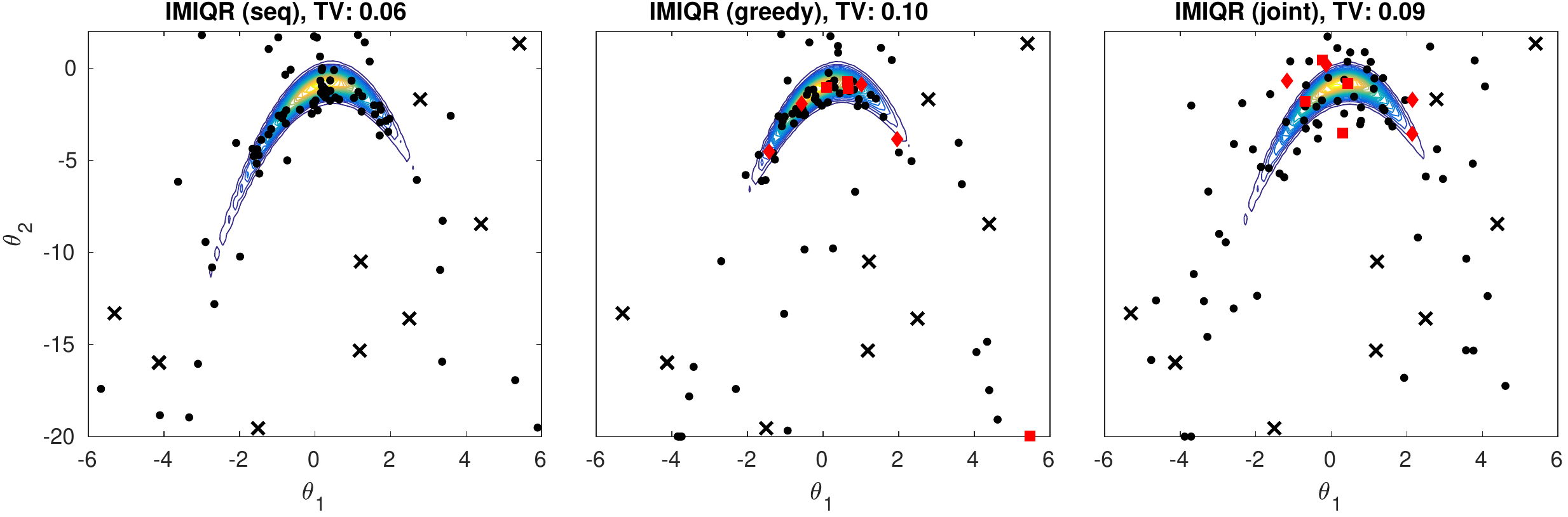}
\caption{The design locations for \miiqr{} are shown after $90$ noisy log-likelihood evaluations of the 2D Banana example with noise level $\sigma_n = 1$. The black crosses show the $b_0=10$ initial evaluations, black dots show obtained design points except for the last two batches, and the red squares and diamonds show the last two batches.}
\label{fig:acq_points_demo_part2}
\end{figure*}

We investigate the effect of batch size $b$ in the greedy batch \maxiqr{} and \maxiqr{} algorithms in Figure~\ref{fig:synth_res_batch}.
In general, the convergence speed of both methods scales well as a function of $b$ and $b=10$ already yields useful improvements. However, increasing $b$ over $40$ would improve the results only slightly. Greedy batch \miiqr{} works overall better than batch \maxiqr{}. The variability in the posterior approximations produced by \miiqr{} is small in all cases unlike for \maxiqr{} which occasionally produced poor approximations (not shown for clarity). 
%
In the \appe{} \ref{subsec:expextratwod} we compare \eiv{} and \miiqr{} in 2D. 
These results show that the greedy \miiqr{} batch-sequential strategy outperforms the corresponding \eiv{} strategy although their difference is small in the corresponding sequential cases. This suggests that, even when $\sigma_n$ is small so that the variance in \eiv{} serves as a reasonable measure of uncertainty and the sequential \eiv{} works similarly to \miiqr{}, the greedy batch median-based \miiqr{} design strategy better mimics the sequential decisions than \eiv{}.

\begin{figure*}[!htbp]
\centering
\includegraphics[width=0.95\textwidth]{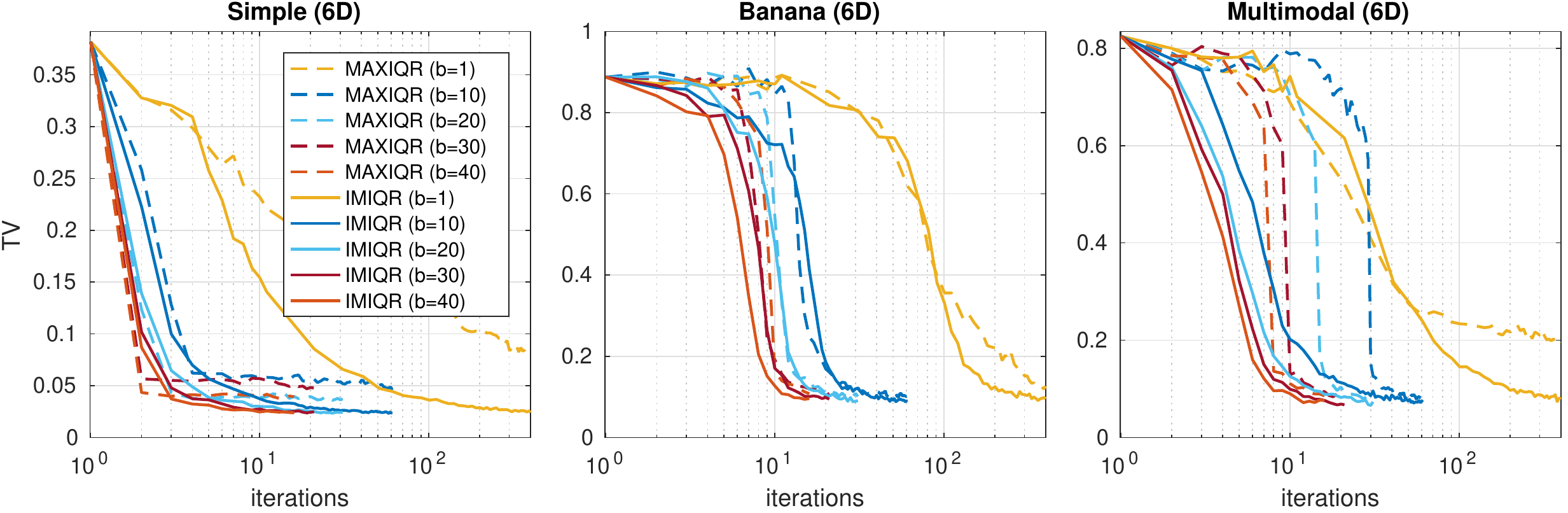}
\caption{Results with greedy batch strategies and varying batch size $b$ for the 6D toy models. For each method, the median TV computed over $50$ repeated simulations is shown. 
The x-axis is truncated after $\iter=400$ iterations to ease visualisation.}
\label{fig:synth_res_batch}
\end{figure*}

\subsection{Simulation models} \label{sec:real_example}

We perform experiments with three benchmark problems used previously in the ABC literature. Two of these are shown here and the third one in the \appe{} \ref{subsec:lorenz}. 
While the proposed methodology is particularly useful for expensive simulation models, we however consider only relatively cheap models as this allows to repeat the computations many times with different realisations of randomness to assess the variability and robustness, and to conduct accurate comparisons to reasonable ground truth posteriors. Nevertheless, these experiments serve as examples of challenging real-world inference scenarios where the GP and SL modelling assumptions do not hold exactly. 
In each problem, we set the unknown parameter of the simulation model to a value used previously in the literature and generated one data set from the simulation model using this ``true'' parameter. The posterior used as the ground truth was computed using \slmcmc{}. Multiple chains each with length $10^6$ were used to ensure that the variability due to Monte Carlo error was small.

\subsubsection{Ricker model} \label{sec:ricker}

We first consider the Ricker model presented in \citet{Wood2010}. In this model $N_t$ denotes the number of individuals in a  population at time $t$ which evolves according to the discrete time stochastic process
$
N_{t+1} = r N_{t} \exp(-N_{t} + \epsilon_t), 
$
for $t=1,\ldots,T$, where $\epsilon_t \simiid \Normal(0,\sigma^2_{\epsilon})$. The initial population size is $N_{0}=1$. 
It is assumed that only a noisy measurement $x_t$ of the population size $N_t$ at each time point is available with the Poisson observation model
$
x_t \cond N_{t}, \phi \sim \Poi(\phi N_{t}).
$
Given data $\Bx = (x_t)_{t=1}^T$, the goal is to infer the three parameters $\Btheta=(\log(r),\phi,\sigma_{\epsilon})$.
%
We use the uniform prior $(\log(r),\phi,\sigma_{\epsilon}) \sim \Unif([3,5]\times[4,20]\times[0,0.8])$. The same $13$ summary statistics as in \citet{Wood2010,Gutmann2016,Price2018} are used to compute log-SL evaluations. The number of repeated simulations is fixed to $N=100$. The ``true'' parameter to be estimated is $\Btheta_{\text{true}}=(3.8,10,0.3)$ and it is used to generate the observed data with length $T=50$. The initial training data size is $b_0=30$ and the additional budget of simulations is $600$ so that the total budget is $630$ SL evaluations corresponding $63000$ simulations. The integrals of \eiv{} and \miiqr{} are approximated using IS and $\sigma_n^2$ is estimated using bootstrap as described in Section \ref{sec:num_details}. 

Figure~\ref{fig:ricker} shows the results. We see that the \eiv{} and \maxv{} strategies perform poorly. These strategies tend to evaluate where the variance of the posterior is high, although as discussed in Section \ref{sec:acq}, these do not necessarily correspond to the regions with non-negligible likelihood. In fact, the magnitude of the log-likelihood and its noise variance $\sigma_n^2$ grow fast near the boundaries of the parameter space where the chaotic nature of the model also makes the log-likelihood surface irregular which further causes difficulties with GP modelling. 
The \miiqr{} method again produces the best posterior approximations which are comparable to the true SL posterior. Some examples are shown in Figure~\ref{fig:ricker}b-d. Also, \maxiqr{} method works well on average but it produces less coherent results than \miiqr{} which is likely the result of its exploitative nature. In addition, unexpectedly, the greedy \maxiqr{} method performs poorly. The probable reason is that the batch evaluations become too diverse having many evaluations in the boundary which leads to poor GP fitting and subsequently poor future designs. 
However, the robust batch-sequential \miiqr{} method with $b=5$ works as expected producing useful improvement to the convergence speed as compared to sequential \miiqr{}. 

\begin{figure*}[htbp]
\centering
\includegraphics[width=1\textwidth]{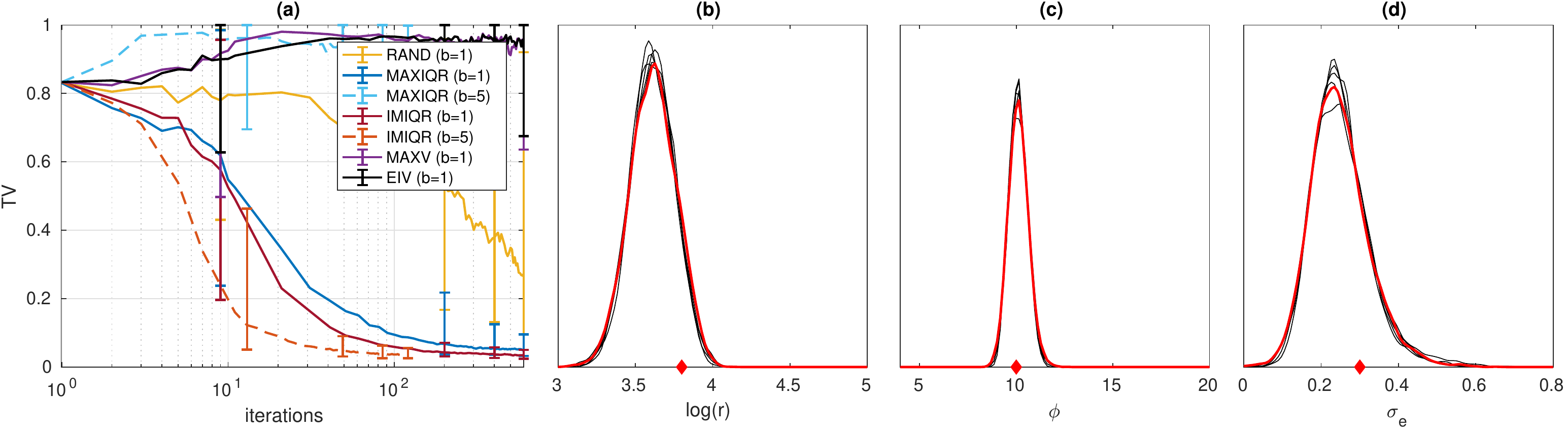}
\caption{Results for the Ricker model. (a) The median TV and $90$\% variability interval over $100$ repeated runs of the algorithms. (b-d) Estimated posterior marginal densities (black) shown for five typical runs of the algorithm with the greedy batch \miiqr{} strategy. The ground truth computed with SL-MCMC (red) and the true parameter value (red diamond) are also shown for comparison.}
\label{fig:ricker}
\end{figure*}

\subsubsection{g-and-k model} \label{sec:gk}

We consider the g-and-k distribution as in \citet{Price2018}. The g-and-k model is a flexible probability distribution defined via its quantile function
\begin{equation}
Q(\Phi^{-1}(p);\Btheta) = a + b\left( 1 + c\frac{1-\exp(-g\Phi^{-1}(p))}{1+\exp(-g\Phi^{-1}(p))} \right)(1+(\Phi^{-1}(p))^2)^k \Phi^{-1}(p), \label{eq:gk_eq}
\end{equation}
where $a,b,c,g$ and $k$ are parameters and $p\in[0,1]$ is a quantile. 
%
We fix $c=0.8$ and estimate the parameters $\Btheta = (a,b,g,k)$ using a uniform prior $\pi(\Btheta) = \Unif([2.5, 3.5]\times[0.5, 1.5]\times[1.5, 2.5]\times[0.3, 0.7])$. We use the same four summary statistics as \citet{Price2018} who fitted an auxiliary model, skew \emph{t}-distribution, to the set of samples generated from \eqref{eq:gk_eq} using maximum likelihood, and took the resulting skew \emph{t} score vector at the ML estimate as the summary statistic. Although there are only $4$ summary statistics, we again use $N=100$. We use the same settings as for the Ricker model except that the initial design is increased to $b_0=40$ so that the total budget is $640$ SL evaluations. The true value of the parameter is chosen to be $\Btheta_{\text{true}} = (3,1,2,0.5)$. 

Overall, the results in Figure~\ref{fig:gk} are similar to those of the Ricker model. However, the larger parameter space slows down the convergence speeds initially as expected, as compared to the Ricker model. Low dimension of the summary statistic and the moderately large value $N=100$ cause the log-likelihood evaluations to be quite accurate near the modal area of the likelihood ($\sigma_n(\Btheta_{\text{true}}) \approx 0.15$) and we expect that smaller $N$ might be already enough. However, using $N=100$ ensures accurate variance estimates using the bootstrap. 
%
%
While \maxiqr{} strategy works almost as well as \miiqr{} on average, it completely fails in some individual repeated experiments producing long variability intervals in Figure~\ref{fig:gk} leaving \miiqr{} as the only successful method. 

\begin{figure*}[!htb]
\centering
\includegraphics[width=1\textwidth]{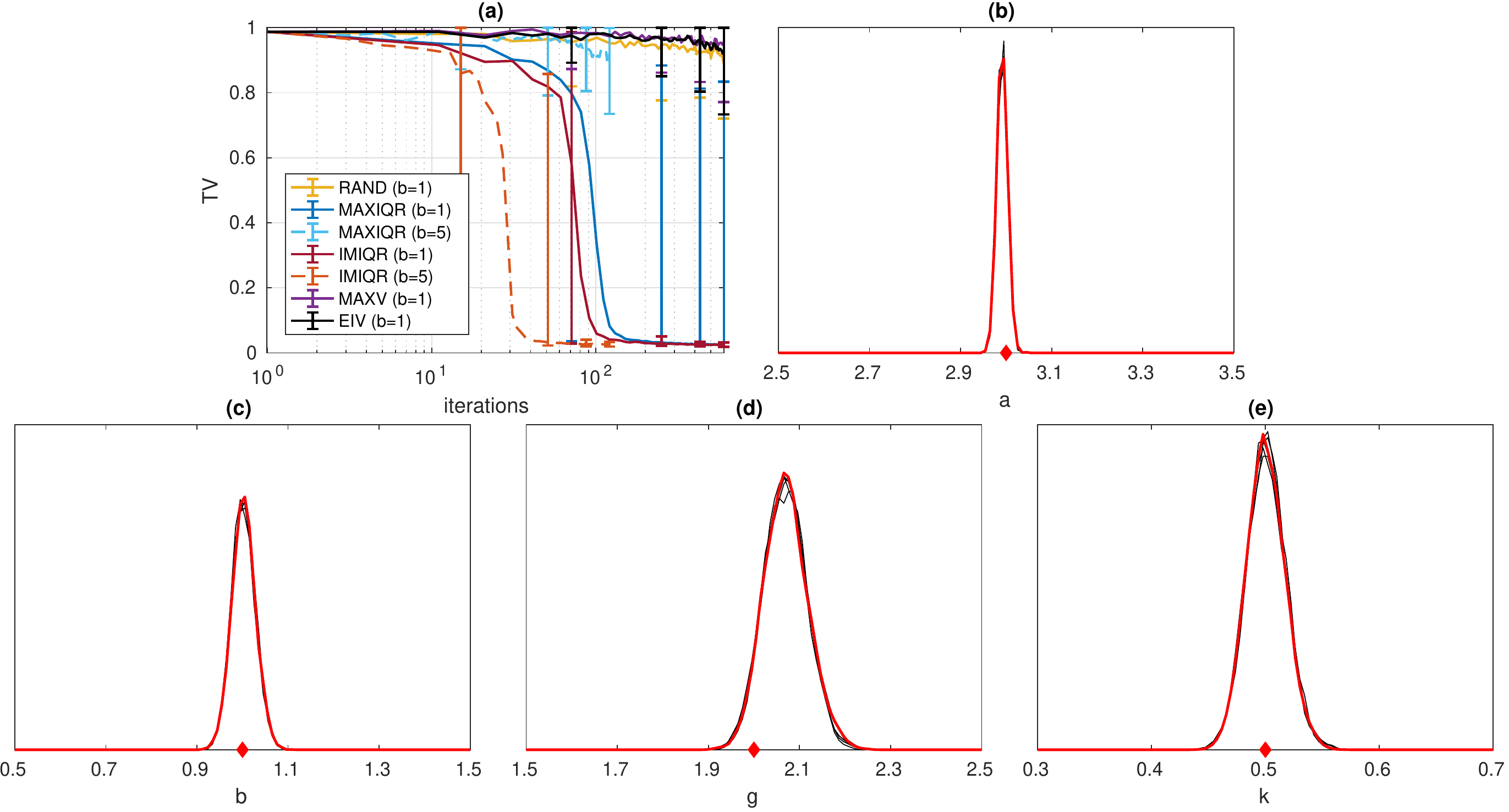}
\caption{(a) Results for the g-and-k model. (a) Median TV and $90$\% variability interval over $100$ repeated runs. (b-e) Some estimated posterior marginal densities illustrated as in Figure~\ref{fig:ricker}.}
\label{fig:gk}
\end{figure*}

\subsubsection{Effect of batch size} \label{sec:res_real_batch}

As the last experiment, we investigate the improvements brought by the batch-sequential \miiqr{} strategy in the case of real-world simulation models. We use the Ricker and g-and-k models from the previous subsections. The experiment details are the same except that we consider only \miiqr{} strategy with several batch sizes $b\in\{2,5,10,20,30\}$. 
%
The results in Figure~\ref{fig:extra_batch} show that, on average, the greedy batch-sequential \miiqr{} with batch sizes up to $30$ produces as good approximations as the corresponding sequential strategy. The convergence speed is also improved almost linearly. 

However, the variability in the quality of the estimated posteriors increases with larger batch sizes when the total budget of simulations is kept fixed. While most of the repeated runs of the algorithm have converged to excellent approximations in all cases as seen in Figure~\ref{fig:extra_batch}, there were some individual runs where the algorithm did not yet converge when the budget was used. The posterior estimate at the final iteration is often quite poor in these cases. Most of these happen with Ricker model when $b\geq 20$ and with g-and-k model when $b = 30$. 
However, this behaviour is not surprising: When $b$ is large, the complete batch is constructed using the same limited information which necessarily produces occasional poor batches providing little information. Furthermore, the importance density in \eqref{eq:is_approx} is likely to get worse when the batch size is increased and cause the last points in the batch to be less useful. It is thus inevitable that the batch size should not be chosen too large. Nevertheless, it is seen that batch size $b=10$ already produces substantial gains (especially considering that further parallelisation is often possible with respect to $N$) and produces consistently accurate posterior approximations.

\begin{figure*}[!htbp]
\centering
\includegraphics[width=0.9\textwidth]{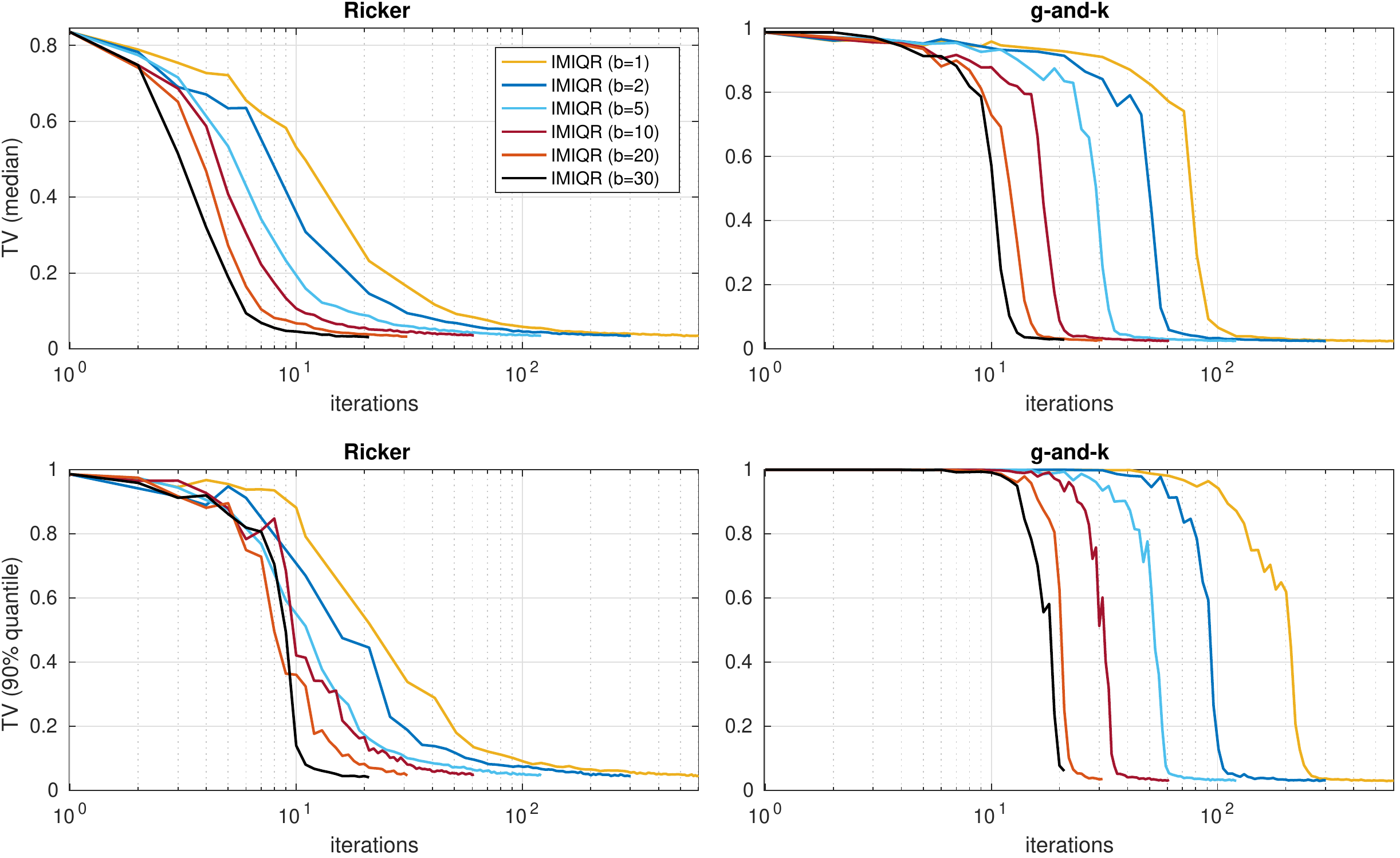}
\caption{Results of the greedy batch strategy \miiqr{} with various batch sizes $b$ for Ricker and g-and-k models. Top row shows the median TV 
over $100$ repeated experiments and the bottom row shows the corresponding 90\% quantile. 
}
\label{fig:extra_batch}
\end{figure*}

\section{Discussion and conclusions} \label{sec:discussion}

If only a limited number of noisy log-likelihood evaluations can be computed, standard techniques such as MCMC become difficult to use for Bayesian inference. To tackle the problem, we constructed a hierarchical GP surrogate model for the noisy log-likelihood and discussed properties of the resulting estimators of the (unnormalised) posterior. 
We developed two batch-sequential strategies (\eiv{} and \miiqr{}) based on Bayesian decision theory, to (semi-)optimally select the next evaluation locations and to parallelise the costly simulations. We also considered heuristic design strategies (\maxv{} and \maxiqr{}). We provided some theoretical analysis: We derived an approximate bound for the greedy optimisation of the batch \miiqr{} method using the concept of weak submodularity, showed a connection between the UCB (a common BO method) and the \maxiqr{} strategy, and between batch \maxiqr{} and the local penalisation method by \citet{Gonzalez2016}. The proposed methods were investigated experimentally. 

The \miiqr{} strategy was found to be robust both to violations of the GP surrogate model assumptions and to the heavy-tails of the resulting distributions. Unlike the other design strategies, it consistently produced posterior approximations comparable to the ground truth. Greedy batch-sequential \miiqr{} strategy was found to be highly useful to parallelise the potentially expensive simulations. In our experiments it produced substantial, sometimes even linear, speed improvements for batch sizes $b\lesssim 20$. 
We thus recommend the \miiqr{} strategy.  
In general we were able to obtain useful posterior approximations with $10,000$ to $20,000$ simulations that can be easily parallellised. This is considerably less than e.g.~using (pseudo-marginal) MCMC requiring typically at least tens of thousands of iterations corresponding to millions of simulations, careful convergence assessment and tuning of the proposal density. 
Another important observation was that the heuristic strategies that evaluate where the current uncertainty is highest, despite their small computational cost and good performance in earlier studies with deterministic evaluations, worked poorly with noisy log-likelihood evaluations.

Similarly to other GP surrogate techniques such as BO, fitting the GP and finding the next evaluation locations by optimising the design criterion is however not free. 
Our unoptimised MATLAB implementations of \maxv{} and \maxiqr{} are fast, but optimisation of the \eiv{} and \miiqr{} design criteria takes a couple of seconds in 2D and around $20$ to $80$ seconds per parameter in 3D and 4D. This means that the proposed algorithm is useful when the simulation time is several seconds or more, which is however true with many real-world simulation models. 
Furthermore, the quality of the posterior approximation also depends on the choice of the surrogate model. We used the same GP model in all of our experiments with no problem-specific tuning, which already produced good results. However, some problems would certainly benefit from further adjustment and incorporation of domain knowledge. For example, if the likelihood is expected to be flat, a GP prior with a constant mean function might be appropriate.

In this work, similarly to \citet{Rasmussen2003,Wilkinson2014,Kandasamy2015,Gutmann2016,Drovandi2018}, we built our surrogate GP model for the log-likelihood. 
An alternative way would be to model the summary statistics. \citet{Meeds2014} used such an approach but they assumed that the summary statistics are independent. However, modelling the scalar-valued log-likelihood is simpler and our approach also applies as such to non-ABC scenarios with exact (but potentially expensive) log-likelihood evaluations as in \citet{Osborne2012,Kandasamy2015,Wang2018,Acerbi2018}. \citet{Gutmann2016,Jarvenpaa2018_aoas,Jarvenpaa2018_acq} modelled the discrepancy between simulated and observed data with a GP and obtained reasonable posterior approximations with only a few hundred model simulations. Here we need $N$ repeated simulations just to compute the log-likelihood for a single parameter value, which can be seen as the price of not having to specify an explicit discrepancy measure and the ABC tolerance. 

We see several avenues for future research. 
%
The consistency and convergence rates of our algorithms could be investigated theoretically. Some work towards that direction has been done by \citet{Bect2018,Stuart2018}. 
Adaptive control of the number of repeated simulations $N$ could likely be used to further reduce the number of simulations required, possibly as in \citet{Picheny2013}. 
Some simulation models may behave unexpectedly near the boundaries of the parameter space violating GP model assumptions as we saw with the Ricker model in Section \ref{sec:experiments}. Similarly, situations where the prior is significantly more diffuse than the posterior may be unsuitable for our approach that relies on a global GP surrogate. Consequently, it would be useful to learn adaptively not only where to evaluate next but also which parameter regions to rule out completely. This could be done as in \citet{Wilkinson2014} or possibly by adapting ideas from the constrained BO literature \citep{Gardner2014,Sui2015}.

\subsection*{Acknowledgements}

The authors are grateful to an associate editor and two anonymous reviewers for their constructive feedback that helped to significantly improve this article. 
This work was funded by the Academy of Finland (grants no. 286607 and 294015 to PM). We acknowledge the computational resources provided by Aalto Science-IT project.

\newpage
{
\subsection*{References}
\renewcommand{\section}[2]{}
\bibliography{bib/refs4}
\bibliographystyle{plainnat}
}

\appendix

\section{Proofs} \label{app:proofs}

We first show that the (marginal) median minimises the expected $L^1$ loss defined in Section \ref{sec:estimators} and then derive the corresponding Bayes risk in \eqref{eq:loneloss}. 
The expected $L^1$ loss can be written similarly to \eqref{eq:bayes_risk_ltwo} but with the absolute value in place of the quadratic term. It then again follows from the basic results of Bayesian decision theory that the integrand, and thus also the original formula, is minimised when $\tilded(\Btheta) = \med_{\lik\cond D_{\xt}}(\tildepiapprox_{\lik}(\Btheta))=\pi(\Btheta)\exp(m_{\xt}(\Btheta))$ for (almost all) $\Btheta\in\Theta$. 

To derive the formula for the Bayes risk, we fix $\Btheta$ and shorten the  notation so that $\lik_{\Btheta}=\lik(\Btheta), m_{\Btheta}=m_{\xt}(\Btheta)$ and $s_{\Btheta}=s_{\xt}(\Btheta)$. Then we obtain
\begin{align}
\mean_{\lik\cond D_{\xt}}(|\e^{\lik_{\Btheta}}-\e^{m_{\Btheta}}|) 
&= \int_{-\infty}^{\infty}|\e^{\lik_{\Btheta}}-\e^{m_{\Btheta}}|\frac{1}{\sqrt{2\pi s_{\Btheta}^2}}\e^{\frac{1}{2s_{\Btheta}^2}(\lik_{\Btheta} - m_{\Btheta})^2} \ud\lik_{\Btheta} \\
\begin{split}
&= \int_{-\infty}^{m_{\Btheta}} (\e^{m_{\Btheta}}-\e^{\lik_{\Btheta}})\frac{1}{\sqrt{2\pi s^2_{\Btheta}}}\e^{\frac{1}{2s_{\Btheta}^2}(\lik_{\Btheta} - m_{\Btheta})^2} \ud\lik_{\Btheta} \\ 
%
&\myquad+ \int_{m_{\Btheta}}^{\infty} (\e^{\lik_{\Btheta}}-\e^{m_{\Btheta}})\frac{1}{\sqrt{2\pi s_{\Btheta}^2}}\e^{\frac{1}{2s_{\Btheta}^2}(\lik_{\Btheta} - m_{\Btheta})^2} \ud\lik_{\Btheta} 
\end{split} \\
&= \frac{\e^{m_{\Btheta}}}{2} - \int_{-\infty}^{m_{\Btheta}} \frac{\e^{\lik_{\Btheta}}}{\sqrt{2\pi s_{\Btheta}^2}}\e^{\frac{1}{2s_{\Btheta}^2}(\lik_{\Btheta} - m_{\Btheta})^2} \ud\lik_{\Btheta}
+ \int_{m_{\Btheta}}^{\infty} \frac{\e^{\lik_{\Btheta}}}{\sqrt{2\pi s_{\Btheta}^2}}\e^{\frac{1}{2s_{\Btheta}^2}(\lik_{\Btheta} - m_{\Btheta})^2} \ud\lik_{\Btheta} - \frac{\e^{m_{\Btheta}}}{2} \\
&= 2 \int_{m_{\Btheta}}^{\infty} \frac{\e^{\lik_{\Btheta}}}{\sqrt{2\pi s_{\Btheta}^2}}\e^{\frac{1}{2s_{\Btheta}^2}(\lik_{\Btheta} - m_{\Btheta})^2} \ud\lik_{\Btheta} - \mean_{\lik\cond D_{\xt}}(\e^{\lik_{\Btheta}}) \\
&= 2 \e^{m_{\Btheta}+s_{\Btheta}^2/2} \underbrace{\int_{m_{\Btheta}}^{\infty} \frac{1}{\sqrt{2\pi s_{\Btheta}^2}} \e^{\frac{1}{2s_{\Btheta}^2}(\lik_{\Btheta} - (m_{\Btheta} + s_{\Btheta}^2))^2} \ud\lik_{\Btheta}}_{=1-\Phi\left( \frac{m_{\Btheta}-(m_{\Btheta}+s_{\Btheta}^2)}{s_{\Btheta}} \right) = \Phi(s_{\Btheta})} - \underbrace{\mean_{\lik\cond D_{\xt}}(\e^{\lik_{\Btheta}})}_{=\e^{m_{\Btheta}+s_{\Btheta}^2/2}} \\
&= \e^{m_{\Btheta}+s_{\Btheta}^2/2}(2\Phi(s_{\Btheta})-1),
\end{align}
where on the fifth line we have completed the square and used the moment-generating function $M_{z}(t) \eqdef \mean(\e^{tz}) = \exp(t\mu + \sigma^2t^2/2)$ of the Gaussian distribution $z\sim\Normal(\mu,\sigma^2)$. 
The desired result follows by multiplying the above with prior density $\pi(\Btheta)$ and integrating the resulting formula over $\Theta$. 
%

Next we show a result from matrix algebra that we need in the following several times.
%
\begin{lemma} \label{lemma:0}
Suppose $\BX_1,\BX_2,\BA,\BB,\BC,\BD,\BY_1$ and $\BY_2$ are such matrices that the equation below is well-defined, that is, the sizes of the matrices are correct and all the required inverses exist. Then
\begin{align}
\begin{bmatrix}
\BX_1 & \BX_2
\end{bmatrix}\!\begin{bmatrix}
\BA & \BB \\ \BC & \BD
\end{bmatrix}^{-1}\!\begin{bmatrix}
\BY_1 \\ \BY_2
\end{bmatrix} 
= \BX_1\BA^{-1}\BY_1 - (\BX_2-\BX_1\BA^{-1}\BB)[\BD-\BC\BA^{-1}\BB]^{-1}(\BY_2-\BC\BA^{-1}\BY_1).
\end{align}
\end{lemma}
%

\begin{proof}[Proof of Lemma \ref{lemma:1}]
The proof is rather straightforward but laborious. To shorten notation, we rename the training data $\Btheta_{1:t}$ as $\Btheta_0$ and the test point $\Btheta$ as $\Btheta_{\bu}$ and we change the subscripts of various matrices appearing in the GP formulas similarly. With this compact notation, the GP formulas in \eqref{eq:gp_mean} and \eqref{eq:gp_cov} become
\begin{align}
    m_0(\Btheta_{\bu}) &= \underbrace{\Bk_{\bu 0} \BK_0^{-1} \By_0}_{\eqdef \cm_0(\Btheta_{\bu})} 
    + \BR_0\T(\Btheta_{\bu}) \bar{\Bgamma}_0, \label{eq:gp_mean2} \\ 
    %
    c_0(\Btheta_{\bu},\Btheta_{\star}) &= \underbrace{\Bk_{{\bu}\star} - \Bk_{{\bu}0} \BK_0^{-1} \Bk_{0\star}}_{\eqdef \cn_0(\Btheta_{\bu},\Btheta_{\star})} 
    + \BR_0\T(\Btheta_{\bu})[\underbrace{\BB^{-1} + \BH_0\BK_0^{-1}\BH_0\T}_{\eqdef\BW_0 }]^{-1} \BR_0(\Btheta_{\star}). \label{eq:gp_cov2}
\end{align}
%
We also define $\BLambda^* = \diag(\sigma_n^2(\Btheta^*_{1}),\ldots,\sigma_n^2(\Btheta^*_{b}))$. 

Given the full training data 
$D_{0*} = D_0\cup D^* = \{(y_{0i},\Btheta_{0i})\}_{i=1}^t\cup\{(y^*_{i},\Btheta^*_{i})\}_{i=1}^b$, using Lemma \ref{lemma:0} analogously as in the proof of Proposition 3.2 in \citet{Jarvenpaa2018_acq} but acknowledging that $\Btheta^*$ contains $b$ points and $c_0(\Btheta^*)$ is thus a $b\times b$ matrix, one obtains the following GP formulas when the GP prior mean function is zero
\begin{align}
\cm_{0*}(\Btheta_{\bu}) &= \cm_0(\Btheta_{\bu}) + \cn_0(\Btheta_{\bu},\Btheta^*)[\cn_0(\Btheta^*) + \BLambda^*]^{-1}(\By^* - \cm_0(\Btheta^*)), \label{eq:gp_mean_zero} \\
\cs_{0*}^2(\Btheta_{\bu}) &= \cs_{0}^2(\Btheta_{\bu}) - \cn_0(\Btheta_{\bu},\Btheta^*)[\cn_0(\Btheta^*) + \BLambda^*]^{-1}\cn_0(\Btheta^*,\Btheta_{\bu}),   \label{eq:gp_var_zero}
\end{align}
where we have denoted $\cs_{0*}^2(\Btheta_{\bu}) \eqdef \cn_{0*} (\Btheta_{\bu},\Btheta_{\bu})$ similarly as before. 


It remains to handle the extra terms in \eqref{eq:gp_mean2} and \eqref{eq:gp_cov2} which are due to the non-zero GP prior mean function assumption. We first compute using Lemma \ref{lemma:0} that 
\begin{equation}\begin{aligned}
\BR_{0*}(\Btheta_{\bu}) 
&= \BH_{\bu} - \begin{bmatrix}
\BH_0 & \BH_*
\end{bmatrix}\!\begin{bmatrix}
\BK_0 & \Bk_{0*} \\ \Bk_{*0} & \BK_*
\end{bmatrix}^{-1}\!\begin{bmatrix}
\Bk_{0{\bu}} \\ \Bk_{*{\bu}}
\end{bmatrix} \\
&= \BH_{\bu} - \BH_0\BK_0^{-1}\Bk_{0{\bu}} - (\BH_*-\BH_0\BK_0^{-1}\Bk_{0*})[\BK_*-\Bk_{*0}\BK_0^{-1}\Bk_{0*}]^{-1}(\Bk_{*{\bu}}-\Bk_{*0}\BK_0^{-1}\Bk_{0{\bu}}) \\
&= \BR_0(\Btheta_{\bu}) - \BR_0(\Btheta^*)[\cn_0(\Btheta^*) + \BLambda^*]^{-1}\cn_0(\Btheta^*,\Btheta_{\bu}). \label{eq:R01}
\end{aligned}\end{equation}
A similar computation shows that
\begin{align}
\BH_{0*}\BK_{0*}\By_{0*} + \BB^{-1}\Bb = \BH_0\BK_0^{-1}\By_0 + \BB^{-1}\Bb + \BR_0(\Btheta^*)[\cn_0(\Btheta^*) + \BLambda^*]^{-1}(\By^*-\cm_0(\Btheta^*)). \label{eq:HKybB}
\end{align}
Similarly we obtain also the formula
\begin{align}
\BW_{0*} &= \BB^{-1} + \BH_{0*}\BK_{0*}^{-1}\BH_{0*}\T = \BW_0 + \BR_0(\Btheta^*)[\cn(\Btheta^*) + \BLambda^*]^{-1}\BR_0\T(\Btheta^*)
\end{align}
from which we further obtain by using the matrix inversion lemma that
\begin{align}
\BW_{0*}^{-1} = \BW_0^{-1} - \BW_0^{-1}\BR_0(\Btheta^*)[c_0(\Btheta^*) + \BLambda^*]^{-1}\BR_0\T(\Btheta^*)\BW_0^{-1}. \label{eq:W01inv}
\end{align}
Using \eqref{eq:R01}, \eqref{eq:W01inv} and \eqref{eq:gp_cov2}, as well as some straightforward manipulations, we obtain the formulas
\begin{align}
\BR_{0*}\T(\Btheta_{\bu})\BW_{0*}^{-1} 
&= \BR_0\T(\Btheta_{\bu})\BW_0^{-1} - c_0(\Btheta_{\bu},\Btheta^*)[c_0(\Btheta^*) + \BLambda^*]^{-1}\BR_0\T(\Btheta^*)\BW_0^{-1}, \label{eq:RinvW} \\
\begin{split}
\BR_{0*}\T(\Btheta_{\bu})\BW_{0*}^{-1}\BR_{0*}(\Btheta_{\bu}) 
&= \BR_0\T(\Btheta_{\bu})\BW_0^{-1}\BR_0(\Btheta_{\bu}) + \cn_0(\Btheta_{\bu},\Btheta^*)[\cn_0(\Btheta^*) + \BLambda^*]^{-1}\cn_0(\Btheta^*,\Btheta_{\bu}) \\
&\myquad- c_0(\Btheta_{\bu},\Btheta^*)[c_0(\Btheta^*) + \BLambda^*]^{-1}c_0(\Btheta^*,\Btheta_{\bu}). 
\end{split} \label{eq:RinvWR}
\end{align}

Putting the results in \eqref{eq:gp_mean_zero}, \eqref{eq:RinvW} and \eqref{eq:HKybB} together and after some additional straightforward simplifications, we see that
\begin{equation}\begin{aligned}
m_{0*}(\Btheta_{\bu}) &= \cm_{0*}(\Btheta_{\bu}) + \BR_{0*}\T(\Btheta_{\bu})\bar{\Bgamma}_{0*} \\
&= m_{0}(\Btheta_{\bu}) + c_0(\Btheta_{\bu},\Btheta^*)[c_0(\Btheta^*) + \BLambda^*]^{-1}(\By^* - m_0(\Btheta^*)). \label{eq:m01eq}
\end{aligned}\end{equation}
By the assumption $\By^*\cond\Btheta^*,D_0 \sim \Normal(m_0(\Btheta^*),c_0(\Btheta^*,\Btheta^*) + \BLambda^*)$ and using \eqref{eq:m01eq} we see that $m_{0*}(\Btheta_{\bu})\cond\Btheta^*,D_0 \sim \Normal(m_{0}(\Btheta_{\bu}),c_0(\Btheta_{\bu},\Btheta^*)[c_0(\Btheta^*) + \BLambda^*]^{-1}c_0(\Btheta^*,\Btheta_{\bu}))$. Thus \eqref{eq:gp_mean_lookahead} holds. 


The variance formula now follows similarly. Using \eqref{eq:gp_var_zero} and \eqref{eq:RinvWR} we obtain
\begin{equation}\begin{aligned}
s_{0*}^2(\Btheta_{\bu}) &= \cs_{0*}^2(\Btheta_{\bu}) + \BR_{0*}\T(\Btheta_{\bu})\BW_{0*}^{-1}\BR_{0*}(\Btheta_{\bu}) \\
&= s_0^2(\Btheta_{\bu}) - c_0(\Btheta_{\bu},\Btheta^*)[c_0(\Btheta^*) + \BLambda^*]^{-1}c_0(\Btheta^*,\Btheta_{\bu}),
\end{aligned}\end{equation}
from which the claim follows.
\end{proof}

\begin{proof}[Proof of Lemma \ref{lemma:2}]
(i) If $\BP$ is a permutation matrix that changes the order of the columns of $\Btheta^*$, then it is easy to see that
\begin{align}
\Deltav_{\xt}(\Btheta;\Btheta^*\BP)
&= c_{\xt}(\Btheta,\Btheta^*)\BP\T[\BP c_{\xt}(\Btheta^*,\Btheta^*)\BP\T + \BP\diag(\sigma_n^2(\Btheta_1^*),\ldots,\sigma_n^2(\Btheta_b^*))\BP\T]^{-1}\BP c_{\xt}(\Btheta^*,\Btheta) \\
&= c_{\xt}(\Btheta,\Btheta^*)\BP\T \BP[c_{\xt}(\Btheta^*,\Btheta^*) + \diag(\sigma_n^2(\Btheta_1^*),\ldots,\sigma_n^2(\Btheta_b^*))]^{-1}\BP\T \BP c_{\xt}(\Btheta^*,\Btheta) \\
&= c_{\xt}(\Btheta,\Btheta^*)[c_{\xt}(\Btheta^*,\Btheta^*) + \diag(\sigma_n^2(\Btheta_1^*),\ldots,\sigma_n^2(\Btheta_b^*))]^{-1}c_{\xt}(\Btheta^*,\Btheta) \\
&= \Deltav_{\xt}(\Btheta;\Btheta^*).
\end{align}
%
%
%
%
(ii) This claim follows from (iv) since the rightmost term of \eqref{eq:batch_shortcut} is clearly non-negative. \\
(iii) 
Straightforward computations show that
\begin{align}
\Deltav_{\xt}(\Btheta;\Btheta^*) &= 
\begin{bmatrix}
c_{\xt}(\Btheta_1^*,\Btheta) \\ c_{\xt}(\Btheta_2^*,\Btheta)
\end{bmatrix}\T\!\begin{bmatrix}
c_{\xt}(\Btheta_1^*)+\sigma_n^2(\Btheta_1^*) & c_{\xt}(\Btheta_1^*,\Btheta_2^*) \\ c_{\xt}(\Btheta_1^*,\Btheta_2^*) & c_{\xt}(\Btheta_2^*)+\sigma_n^2(\Btheta_2^*)
\end{bmatrix}^{-1}\!\begin{bmatrix}
c_{\xt}(\Btheta_1^*,\Btheta) \\ c_{\xt}(\Btheta_2^*,\Btheta)
\end{bmatrix} \label{eq:iiifirstline} \\
%
&= \frac{1}{\bar{s}^2_{\xt}(\Btheta_1^*)\bar{s}^2_{\xt}(\Btheta_2^*) - c^2_{\xt}(\Btheta_1^*,\Btheta_2^*)}
\begin{bmatrix}
c_{\xt}(\Btheta_1^*,\Btheta) \\ c_{\xt}(\Btheta_2^*,\Btheta)
\end{bmatrix}\T\!\begin{bmatrix}
\bar{s}^2_{\xt}(\Btheta_2^*) & -c_{\xt}(\Btheta_1^*,\Btheta_2^*) \\ -c_{\xt}(\Btheta_1^*,\Btheta_2^*) & \bar{s}^2_{\xt}(\Btheta_1^*)
\end{bmatrix}\!\begin{bmatrix}
c_{\xt}(\Btheta_1^*,\Btheta) \\ c_{\xt}(\Btheta_2^*,\Btheta)
\end{bmatrix} \\
%
\begin{split}
&= \Deltav_{\xt}(\Btheta;\Btheta_1^*) + \Deltav_{\xt}(\Btheta;\Btheta_2^*) 
- \Deltav_{\xt}(\Btheta;\Btheta_1^*) - \Deltav_{\xt}(\Btheta;\Btheta_2^*) \\
%
&\myquad+ \frac{c^2_{\xt}(\Btheta,\Btheta_1^*)\bar{s}^2_{\xt}(\Btheta_2^*) + c^2_{\xt}(\Btheta,\Btheta_2^*)\bar{s}^2_{\xt}(\Btheta_1^*) - 2c_{\xt}(\Btheta,\Btheta_1^*)c_{\xt}(\Btheta,\Btheta_2^*)c_{\xt}(\Btheta_1^*,\Btheta_2^*)}
%
{\bar{s}^2_{\xt}(\Btheta_1^*)\bar{s}^2_{\xt}(\Btheta_2^*) - c^2_{\xt}(\Btheta_1^*,\Btheta_2^*)}, 
\end{split} \\
&= \Deltav_{\xt}(\Btheta;\Btheta_1^*) + \Deltav_{\xt}(\Btheta;\Btheta_2^*) 
+ \frac{c^2_{\xt}(\Btheta_1^*,\Btheta_2^*)c^2_{\xt}(\Btheta,\Btheta_1^*)\bar{s}^2_{\xt}(\Btheta_2^*) + c^2_{\xt}(\Btheta_1^*,\Btheta_2^*)c^2_{\xt}(\Btheta,\Btheta_2^*)\bar{s}^2_{\xt}(\Btheta_1^*)} 
%
{\bar{s}^4_{\xt}(\Btheta_1^*)\bar{s}^4_{\xt}(\Btheta_2^*) - \bar{s}^2_{\xt}(\Btheta_1^*)\bar{s}^2_{\xt}(\Btheta_2^*)c^2_{\xt}(\Btheta_1^*,\Btheta_2^*)} \nonumber \\
%
&\myquad- 2\frac{c_{\xt}(\Btheta,\Btheta_1^*)c_{\xt}(\Btheta,\Btheta_2^*)c_{\xt}(\Btheta_1^*,\Btheta_2^*)\bar{s}^2_{\xt}(\Btheta_1^*)\bar{s}^2_{\xt}(\Btheta_2^*)}
{\bar{s}^4_{\xt}(\Btheta_1^*)\bar{s}^4_{\xt}(\Btheta_2^*) - \bar{s}^2_{\xt}(\Btheta_1^*)\bar{s}^2_{\xt}(\Btheta_2^*)c^2_{\xt}(\Btheta_1^*,\Btheta_2^*)} \\
&= \Deltav_{\xt}(\Btheta;\Btheta_1^*) + \Deltav_{\xt}(\Btheta;\Btheta_2^*)  + r_{\xt}(\Btheta;\Btheta^*_1,\Btheta^*_2). 
\end{align}
(iv) The result follows immediately by applying Lemma \ref{lemma:0} to an equation corresponding to \eqref{eq:iiifirstline} in the proof of (iii) but where one has the matrix $\bar{\BS}_A$ in place of the scalar $c_{\xt}(\Btheta_1^*) + \sigma_n^2(\Btheta_1^*)$.
\end{proof}

\begin{proof}[Proof of Proposition \ref{prop:var}]
We compute
\begin{align}
%
L_{\xt}^{\text{v}}(\Btheta^*)
&= \mean_{\By^*\cond \Btheta^*, D_{\xt}} \int_{\Theta} \pi^2(\Btheta)\e^{2m_{\xtb}(\Btheta;\By^*,\Btheta^*) + s_{\xtb}^2(\Btheta;\Btheta^*)}\left(\e^{s_{\xtb}^2(\Btheta;\Btheta^*)} - 1\right) \ud\Btheta \\
%
&= \int_{\Theta} \pi^2(\Btheta) \mean_{\By^*\cond \Btheta^*, D_{\xt}} \left(\e^{2m_{\xtb}(\Btheta;\By^*,\Btheta^*) + s_{\xtb}^2(\Btheta;\Btheta^*)}\left(\e^{s_{\xtb}^2(\Btheta;\Btheta^*)} - 1\right)\right) \ud\Btheta \\
&= \int_{\Theta} \pi^2(\Btheta) 
\underbrace{\mean_{m_{\xtb}(\Btheta;\Btheta^*)\cond \Btheta^*, D_{\xt}} \left(\e^{2m_{\xtb}(\Btheta;\Btheta^*)}\right)}_{=\e^{2m_{\xt}(\Btheta) + 2\Deltav_{\xt}(\Btheta;\Btheta^*)}}
\e^{s_{\xtb}^2(\Btheta;\Btheta^*)}\left(\e^{s_{\xtb}^2(\Btheta;\Btheta^*)} - 1\right) \ud\Btheta \\
&= \int_{\Theta} \pi^2(\Btheta) \e^{2m_{\xt}(\Btheta) + s^2_{\xt}(\Btheta)}\left(\e^{s^2_{\xt}(\Btheta)} - \e^{\Deltav_{\xt}(\Btheta;\Btheta^*)}\right) \ud\Btheta,
\end{align}
where on the second line we have used Tonelli theorem to change the order of expectation and integration, on the third line we have used Lemma \ref{lemma:1}, and the expectation on the third line is computed using the moment-generating function of the Gaussian distribution. 
\end{proof}

\begin{proof}[Proof of Proposition \ref{prop:iqr}]
Using Lemma \ref{lemma:1}, we see that the pointwise median in the integrand of \eqref{eq:miiqr_prop} is computed as 
$\med_{\By^*\cond \Btheta^*, D_{\xt}} \left(\e^{m_{\xtb}(\Btheta;\By^*, \Btheta^*)}\right)
=
\med_{m_{\xtb}(\Btheta;\Btheta^*)\cond \Btheta^*, D_{\xt}} \left(\e^{m_{\xtb}(\Btheta;\Btheta^*)}\right) = \e^{m_{\xt}(\Btheta)}$ from which the final result follows.
\end{proof}
%

\section{Analysis of the greedy optimisation of design criteria} \label{app:analysis_submodularity}

\subsection{On the monotonicity of design criteria} \label{sec:monotonicity}

We show that \eiv{} and \miiqr{} design criteria are non-increasing functions of the batch size $b$. We also discuss why this does not generally hold for expected integrated IQR (abbreviated \eiiqr{}) which further justifies our choice of \miiqr{} over \eiiqr{}. 

Suppose that $\Btheta_A^*\subseteq \Btheta_B^*$ where we allow $\Btheta_A^*$ to be an empty multiset so that $\Deltav_{\xt}(\Btheta;\emptyset) = 0$. Now using Lemma \ref{lemma:2}, we see that \eiv{} is non-increasing because
\begin{align}
L_{\xt}^{\text{v}}(\Btheta_A^*) - L_{\xt}^{\text{v}}(\Btheta_B^*)
= \int_{\Theta} \pi^2(\Btheta) \e^{2m_{\xt}(\Btheta) + s_{\xt}^2(\Btheta)} \underbrace{\left[ \e^{\Deltav_{\xt}(\Btheta;\Btheta^*_B)} - \e^{\Deltav_{\xt}(\Btheta;\Btheta^*_A)}\right]}_{\geq 0}  \ud \Btheta \geq 0.
\end{align}
Similarly, using the fact that $z\mapsto\sinh(z)$ is an increasing function and recalling that $u=\Phi^{-1}(p_u)>0$, we see that \miiqr{} is non-increasing:
\begin{equation}\begin{aligned}
&\tilde{L}_{\xt}^{\text{IQR}}(\Btheta_A^*) - \tilde{L}_{\xt}^{\text{IQR}}(\Btheta_B^*) 
%
= 2\int_{\Theta} \pi(\Btheta) \e^{m_{\xt}(\Btheta)} \underbrace{\left[ \sinh(us_{\xtb_A}(\Btheta;\Btheta^*_A)) - \sinh(us_{\xtb_B}(\Btheta;\Btheta^*_B))\right]}_{\geq 0}  \ud \Btheta \geq 0.
\end{aligned}\end{equation}
%

We next analyse the \eiiqr{} strategy. 
The design criterion for \eiiqr{}, denoted $L_{\xt}^{\text{IQR,e}}$, is given by 
\begin{align}
L_{\xt}^{\text{IQR,e}}(\Btheta^*) = 2\int_{\Theta} \pi(\Btheta) \e^{m_{\xt}(\Btheta) + \half \Deltav_{\xt}(\Btheta;\Btheta^*)} \sinh(us_{\xtb}(\Btheta;\Btheta^*)) \ud\Btheta.
\end{align}
The proof of this fact is analogous to that of the Proposition \ref{prop:var} and details are thus omitted. As compared to \miiqr{}, the extra term $ \Deltav_{\xt}(\Btheta;\Btheta^*)/2$ appears to the integrand so that $\tilde{L}_{\xt}^{\text{IQR}}(\Btheta^*) \leq L_{\xt}^{\text{IQR,e}}(\Btheta^*)$ holds for all $\Btheta^*$. 

We now briefly analyse \eiiqr{} which is not non-increasing for $b$ in general. Presenting an explicit counterexample is not straightforward so we only heuristically justify why the integrand of \eqref{eq:eiiqr2} can be negative. It is enough to consider the special case where $\Btheta_A^*=\emptyset$. 
We obtain 
\begin{equation}\begin{aligned}
L_{\xt}^{\text{IQR,e}}(\Btheta_A^*) - L_{\xt}^{\text{IQR,e}}(\Btheta_B^*) 
= 2\int_{\Theta} \pi(\Btheta) \e^{m_{\xt}(\Btheta)} \underbrace{\left( \sinh(us_{\xt}(\Btheta)) - \e^{\half \Deltav_{\xt}(\Btheta;\Btheta^*_B)}\sinh(us_{\xtb_B}(\Btheta;\Btheta^*_B))\right)}_{\eqdef \omega(\Btheta;\Btheta^*_B)} \ud \Btheta.
\label{eq:eiiqr2}
\end{aligned}\end{equation}
Suppose $\sigma_n^2(\Btheta)>0$ for all $\Btheta\in\Theta$ and consider $\Btheta\in\Theta$ so that $s_{\xt}(\Btheta)>0$ and  $\pi(\Btheta)>0$. Then one can take $\Btheta^*\in\Theta$ so that there exists $c=c(\Btheta,\Btheta^*)\in(0,1)$ and $\Deltav_{\xt}(\Btheta;\Btheta^*) = c s^2_{\xt}(\Btheta)$. 
We can now write
\begin{align}
\omega(\Btheta;\Btheta^*) &= \sinh(us_{\xt}(\Btheta)) - \e^{\frac{c}{2}s^2_{\xt}(\Btheta)}\sinh(us_{\xt}(\Btheta)\sqrt{1-c}).
\end{align}
If $c$ can be kept fixed but $s_{\xt}(\Btheta) \rightarrow \infty$, then we can see that $\omega(\Btheta;\Btheta^*) \rightarrow -\infty$. That is, if $s_{\xt}(\Btheta)$ is chosen large enough, then $\omega(\Btheta;\Btheta^*) < 0$. It can be further reasoned by continuity that $\omega(\Btheta;\Btheta^*) < 0$ holds in set of nonzero measure around $\Btheta$. 
Simulations suggest that there indeed exists practical scenarios where some choices of $\Btheta^*$ make the expected loss to increase, that is, \eqref{eq:eiiqr2} is negative. In these cases \eiiqr{} algorithm can get stuck to ``safe'' regions of the parameter space because evaluations elsewhere would increase the expected loss. This behaviour produces poor posterior estimates in practice and implies possibly non-convergence of the inference algorithm.

\subsection{Proof of the greedy optimisation bound} \label{subsec_greedyboundproof}

\begin{proof}[Proof of Theorem \ref{thm:miiqr_submod}]
%
The main idea is to first show that $\tilde{U}^{\text{\iqr},a}_{\xt}$ in \eqref{eq:U_miiqr_a} is a weakly submodular set function (see e.g.~\citet{Krause2008,Krause2010} for definition) and then derive the bound using similar reasoning as in \citet{Nemhauser1978} and the observation that (weak) submodularity in their proof is required only for sets with size up to $2b$ instead for all sets. In the following we drop abbreviation ``$\text{\iqr},a$'' from $\tilde{U}^{\text{\iqr},a}_{\xt}$.
%
Let $\Btheta_A\subset\tilde{\Theta}, i \eqdef |\Btheta_A| \leq 2b$ and $\Btheta_j,\Btheta_k\in\tilde{\Theta}\setminus\Btheta_A$. We identify singletons with the corresponding element, that is, we write e.g.~$\Btheta_j$ for $\{\Btheta_j\}$. Then
\begin{align}
    &\tilde{U}_{\xt}(\Btheta_A \cup \Btheta_k) - \tilde{U}_{\xt}(\Btheta_A) - \tilde{U}_{\xt}(\Btheta_A \cup \Btheta_j \cup \Btheta_k) + \tilde{U}_{\xt}(\Btheta_A \cup \Btheta_j) \nonumber \\
    \begin{split}
    &= 2u^2 \int_{\Theta} \pi(\Btheta)\e^{m_{\xt}(\Btheta)} 
    [s^2_{t+i}(\Btheta;\Btheta_A) 
    + s^2_{t+i+2}(\Btheta;\Btheta_A\cup \Btheta_j \cup \Btheta_k) \\
    &\myquad- s^2_{t+i+1}(\Btheta;\Btheta_A\cup \Btheta_j)
    - s^2_{t+i+1}(\Btheta;\Btheta_A\cup \Btheta_k)] \ud \Btheta 
    \end{split}\\
    &= 2u^2 \int_{\Theta} \pi(\Btheta)\e^{m_{\xt}(\Btheta)} 
    [\Deltav_{t+i}(\Btheta;\Btheta_j) 
    + \Deltav_{t+i}(\Btheta;\Btheta_k) 
    - \Deltav_{t+i}(\Btheta; \Btheta_j \cup \Btheta_k)] \ud \Btheta \\
    &= -2u^2 \int_{\Theta} \pi(\Btheta)\e^{m_{\xt}(\Btheta)} r_{t+i}(\Btheta;\Btheta_j,\Btheta_k) \ud \Btheta \\
    &\geq -2u^2 \max_{\substack{\Btheta_A\subset\tilde{\Theta}, |\Btheta_A|=i \leq 2b \\ \Btheta_j,\Btheta_k\in\tilde{\Theta}}} \int_{\Theta} \pi(\Btheta)\e^{m_{\xt}(\Btheta)} r_{t+i}(\Btheta;\Btheta_j,\Btheta_k) \ud \Btheta \\
    &\geq -\max\{0,2u^2\epsilon_{\xt}'\} \\
    &= -\epsilon_{\xt}.
\end{align}
We have thus shown that
\begin{equation}
    \tilde{U}_{\xt}(\Btheta_A \cup \Btheta_k) - \tilde{U}_{\xt}(\Btheta_A) \geq 
    \tilde{U}_{\xt}(\Btheta_A \cup \Btheta_j \cup \Btheta_k) - \tilde{U}_{\xt}(\Btheta_A \cup \Btheta_j) - \epsilon_{\xt}.
\end{equation}
Now assume $|\Btheta_A| \leq b$ and consider set $\Btheta_B = \Btheta_A \cup \{\Btheta_1^B,\ldots,\Btheta_b^B\}\subset\tilde{\Theta}$. Assume also $\Btheta_k \notin \Btheta_B$. If we replace $\Btheta_A$ with $\Btheta_A\cup\{\Btheta_1^B,\ldots,\Btheta_{m-1}^B\}$ and set $\Btheta_j=\Btheta_m^B$ in the above formula, we obtain
\begin{align}
    \begin{split}
    &\tilde{U}_{\xt}(\Btheta_A \cup \{\Btheta_1^B,\ldots,\Btheta_{m-1}^B\} \cup \Btheta_k) - \tilde{U}_{\xt}(\Btheta_A \cup \{\Btheta_1^B,\ldots,\Btheta_{m-1}^B\}) \\
    &\myquad\geq 
    \tilde{U}_{\xt}(\Btheta_A \cup \{\Btheta_1^B,\ldots,\Btheta_{m}^B\} \cup \Btheta_k) - \tilde{U}_{\xt}(\Btheta_A \cup \{\Btheta_1^B,\ldots,\Btheta_{m}^B\}) - \epsilon_{\xt}.
    \end{split}
\end{align}
If we sum up all the above inequalities for $m=1,\ldots,b$, we obtain
\begin{equation}
    \tilde{U}_{\xt}(\Btheta_A \cup \Btheta_k) - \tilde{U}_{\xt}(\Btheta_A) \geq 
    \tilde{U}_{\xt}(\Btheta_B \cup \Btheta_k) - \tilde{U}_{\xt}(\Btheta_B) - b\epsilon_{\xt}. \label{eq:weak_sub_general}
\end{equation}

Next we proceed similarly as the proof of Proposition 11.1 in \citet{Bach2013} but use our weak submodularity condition in \eqref{eq:weak_sub_general}. $\tilde{U}_{\xt}$ is clearly bounded and non-decreasing and $\tilde{U}_{\xt}(\emptyset) = 0$. 
Let $\Btheta_j^G,j=1,\ldots,b$ be the $j$th element selected during the greedy algorithm, $\Btheta^G_{1:j} \eqdef \{\Btheta_1^G,\ldots,\Btheta_j^G\}$ and $\rho_j \eqdef \tilde{U}_{\xt}(\Btheta^G_{1:j}) - \tilde{U}_{\xt}(\Btheta^G_{1:j-1})$. For a given $j$, denote $\Btheta_O\!\setminus\!\Btheta^G_{1:j}=\{\Btheta_1,\ldots,\Btheta_m\}$ so that $m\leq b$. Then
\begin{align}
    &\tilde{U}_{\xt}(\Btheta_O) \nonumber \\
    &\leq \tilde{U}_{\xt}(\Btheta_O \cup \Btheta^G_{1:j-1}) \\
    &= \tilde{U}_{\xt}(\Btheta^G_{1:j-1}) + \sum_{i=1}^{m}[\tilde{U}_{\xt}(\Btheta^G_{1:j-1} \cup \{\Btheta_1,\ldots,\Btheta_i\}) - \tilde{U}_{\xt}(\Btheta^G_{1:j-1} \cup \{\Btheta_1,\ldots,\Btheta_{i-1}\})] \\
    &\leq \tilde{U}_{\xt}(\Btheta^G_{1:j-1}) + \sum_{i=1}^{m}[\tilde{U}_{\xt}(\Btheta^G_{1:j-1} \cup \Btheta_i) - \tilde{U}_{\xt}(\Btheta^G_{1:j-1}) + b\epsilon_{\xt}] \\
    &\leq \tilde{U}_{\xt}(\Btheta^G_{1:j-1}) + m\rho_j + mb\epsilon_{\xt} \\
    &\leq \sum_{i=1}^{j-1}\rho_i + b\rho_j + b^2\epsilon_{\xt}.
\end{align}
If we multiply both sides of the inequality $\sum_{i=1}^{j-1}\rho_i + b\rho_j \geq \tilde{U}_{\xt}(\Btheta_O)-b^2\epsilon_{\xt}$ by $(1-1/b)^{b-j}$ and add the inequalities up for $j=1,\ldots,b$, we obtain\footnote{This last part of our proof is analogous to that of theorem 7 in \citet{Krause2008} and we have corrected the mistake of having $(1-1/k)^{k-1}$ where it should read as $(1-1/k)^{k-i}$ (in their notation where $k$ corresponds our batch size $b$ and $i$ corresponds our index $j$).}
\begin{equation}
    \sum_{i=1}^b(1-1/b)^{b-i}\left( \sum_{j=1}^{i-1}\rho_j + b\rho_i \right) \geq (\tilde{U}_{\xt}(\Btheta_O) - b^2\epsilon_{\xt})\sum_{i=0}^{b-1}(1-1/b)^i.
\end{equation}
After some simplifications, we see that this inequality is equivalent with 
\begin{equation}
    \tilde{U}_{\xt}(\Btheta_G) = \sum_{i=1}^b \rho_i \geq (1-(1-1/b)^b)(\tilde{U}_{\xt}(\Btheta_O) - b^2\epsilon_{\xt}),
\end{equation}
from which the claim follows. 
\end{proof}

\section{Additional implementation details} \label{app:add_details}

In this section we briefly present additional implementation details of the algorithm in Section \ref{sec:num_details}. 
%
We start by pointing out that 
the initial design locations are drawn from the prior $\pi(\Btheta)$ but other techniques such as random or quasi-Monte Carlo designs over $\Theta$ are also possible.

Different estimators of log-SL can be used in our algorithm. 
In fact, the logarithm of \eqref{eq:SL} with plug-in ML estimators in \eqref{eq:SL_ML} produces a biased estimator of the logarithm of the Normal pdf so it might be reasonable to use an unbiased estimator instead. Such an estimator exists and has been used in \citet{Ong2018}. However, we nevertheless used the logarithm of \eqref{eq:SL} with plug-in ML estimates because the resulting estimator was found slightly more robust than the unbiased estimator in \citet{Ong2018} and because both estimators produced similar results in practice. While a systematic comparison of different log-SL estimators was not done, we expect the bias to be small with moderate $N$. Also, in practice, the Gaussianity assumption usually holds only approximately. 
As mentioned in the main text, there also exists an unbiased SL estimator. However, since we are modelling the logarithm of SL, using an unbiased estimator of SL is not advantageous.

To handle the unknown GP hyperparameters $\Bphi$, a plug-in approach is often used where $\Bphi$ is substituted with ML or MAP estimate, see e.g.~\citet{Rasmussen2006}. For fully Bayesian approach, one can use MCMC methods but this is expensive. Uncertainty in $\Bphi$ could be acknowledged also when computing the design criterion as discussed e.g.~in the Section 3.5 of \citet{Jarvenpaa2018_acq}. However, we used the plug-in approach with MAP estimate in our experiments and we re-estimated $\Bphi$ after each iteration as shown on line 7 of Algorithm \ref{alg:gp_sl_alg} using the $\texttt{gp\_optim}$ function of GPstuff 4.7 \citep{Vanhatalo2013}.

A relatively tuning-free adaptive MCMC method by \citet{Haario2006} is used for sampling from $\pi_q$ (and from the posterior estimate on line 25 of Algorithm \ref{alg:gp_sl_alg}). However, because the IS proposal can be multimodal, the sampler may get stuck to a local mode. To alleviate this, we use multiple chains and initialise the sampler at the point with the highest current loss over the training points. This way, even if sampling over the full range of $\Theta$ is not perfect, the loss measures uncertainty in regions where it is relatively large and, consequently, reasonable designs are obtained.  
%
Furthermore, we found that computing this integral very accurately in not necessary for obtaining good designs and ultimately converging to a good posterior approximation.

Several methods for the global optimisation of the design criterion have been used in literature: random search, multistart gradient-based methods, evolution strategies (such as CMA-ES) and partitioning based algorithm DIRECT. Here the optimisation is carried out by first using random search to roughly locate potential optima and then improving the best $10$ points found this way by initialising gradient-based algorithm (MATLAB function \texttt{fmincon}) at these points. Finally, the best point evaluated is reported as the solution. While systematic comparison between optimisers was not done, we observed that this approach produced satisfactory results with reasonable computation time.

We also precompute many quantities in the GP and design criterion formulas to speed-up the optimisation and sampling steps. For example, the Cholesky factor of the full data covariance matrix $\BK_{\xt}$ in \eqref{eq:gp_mean} and \eqref{eq:gp_cov} is precomputed and used for prediction at new locations $\Btheta$. For \eiv{} or \miiqr{}, all the quantities depending only on the integration points $\Btheta^{(j)}$ are precomputed so that only those terms that depend on candidate design points $\Btheta^*$ need to be re-evaluated during the optimisation. Furthermore, in the greedy optimisation, the Lemma \ref{lemma:2} (iv) is used to avoid always inverting the whole covariance matrix in \eqref{eq:gp_dvar} when only the last column $\Btheta_r^*$ of the matrix $\Btheta_{1:r}^*$ is varied. When computing the design criteria, we use logarithms and the logsumexp trick to avoid numerical under- and overflows which otherwise occur when exponentiating the GP mean and variance function values with high magnitude. 

While we consider a fixed budget of simulations $b_0 + \iter_{\text{max}}b$, the algorithm can be prematurely terminated when some suitable stopping criterion is met. For example, if the SL posterior estimate or the value of the design criterion has changed little during a fixed number of the most recent iterations, one could terminate the algorithm. Such stopping criteria have been used by \citet{Acerbi2018,Wang2018}. 

\section{Additional experiments and illustrations} \label{app:add_res}

\subsection{Normality of noisy log-SL evaluations} \label{subsec:logSL_normality}

Figure~\ref{fig:sl_gaussianity} shows the sampling distribution of the log-SL for six benchmark models. 

\begin{figure*}[ht]
\centering
\includegraphics[width=0.9\textwidth]{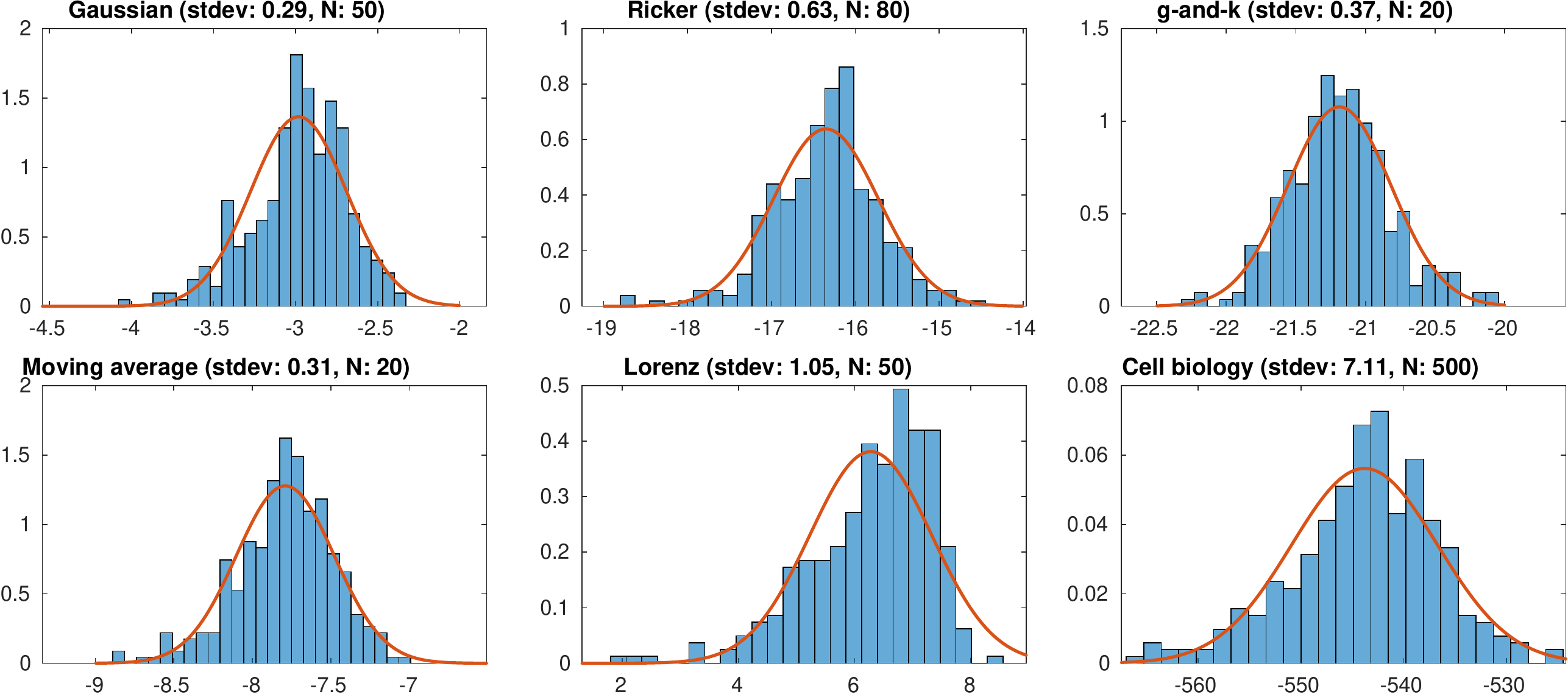}
\caption{Sampling distribution of log-SL for six benchmark models evaluated at their ``true'' parameter values. The densities are approximately Gaussian, although the distribution corresponding to the Lorenz model shows some evidence of skewness. The details of the Ricker, g-and-k and Lorenz model are given in Sections \ref{sec:real_example} and \ref{subsec:lorenz}, ``Gaussian'' is a simple 2D problem where the expectation is estimated, ``Moving average'' is the benchmark model used in \citet{Marin2012}, and the cell biology model is as in \citet{Price2018}.}
\label{fig:sl_gaussianity}
\end{figure*}

\subsection{Details of the toy densities} \label{subsec:app_2d}

The 2D toy densities are shown in Figure~\ref{fig:synth_demo}.
Their log-densities, denoted here as $f_{\textnormal{2D}}$, are defined explicitly so that the log-density for 
the 'Simple' model is obtained as $\lik_{\textnormal{2D}}(\Btheta)=-\Btheta\T \BS^{-1}_{\rho}\Btheta/2$ where $\rho=0.25$, 
for the 'Banana' model as $\lik_{\textnormal{2D}}(\Btheta)=-[\theta_1, \theta_2 + \theta_1^2 + 1] \BS^{-1}_{\rho}[\theta_1, \theta_2 + \theta_1^2 + 1]\T/2$ where $\rho=0.9$ and, finally, 
for the 'Bimodal' model as $\lik_{\textnormal{2D}}(\Btheta)=-[\theta_1, \theta_2^2-2] \BS^{-1}_{\rho}[\theta_1, \theta_2^2-2]\T/2$ where $\rho=0.5$. 
In all of these cases we have $(S_{\rho})_{11}=(S_{\rho})_{22}=1$ and $(S_{\rho})_{12}=(S_{\rho})_{21}=\rho$. 
%
The prior densities for 'Simple', 'Banana' and 'Bimodal' test cases are chosen as $\pi_{\textnormal{2D}}(\Btheta) = \Unif([-16,16]^2)$, $\pi_{\textnormal{2D}}(\Btheta) = \Unif([-6,6]\times[-20,2])$ and $\pi_{\textnormal{2D}}(\Btheta) = \Unif([-6,6]^2)$, respectively. 

The 6D log-densities, denoted here as $f_{\textnormal{6D}}$, are then constructed from the 2D log-densities so that $\lik_{\textnormal{6D}}(\Btheta) = \lik_{\textnormal{2D}}(\Btheta_{1:2}) + \lik_{\textnormal{2D}}(\Btheta_{3:4}) + \lik_{\textnormal{2D}}(\Btheta_{5:6})$. The corresponding priors are chosen as $\pi_{\textnormal{6D}}(\Btheta) = \pi_{\textnormal{2D}}(\Btheta_{1:2})\pi_{\textnormal{2D}}(\Btheta_{3:4})\pi_{\textnormal{2D}}(\Btheta_{5:6})$. That is, the priors for the 6D 'Simple', 'Banana' and 'Multimodal' densities are $\Unif([-16,16]^6), \, \Unif(\bigtimes_{i=1}^3([-6,6]\!\times\![-20,2]))$ and $\Unif([-6,6]^6)$, respectively. 

\begin{figure*}[!htbp]
\centering
\includegraphics[width=0.88\textwidth]{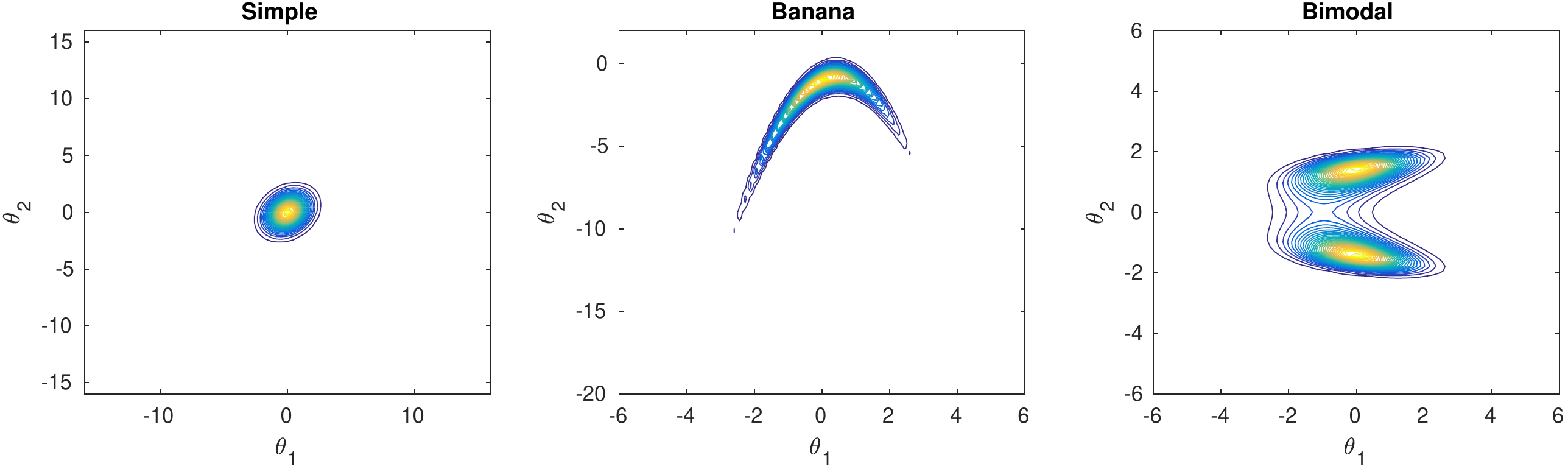}
\caption{Illustration of the 2D toy densities which coincide with the 2D marginals of $\Btheta_{1:2},\Btheta_{3:4}$ and $\Btheta_{5:6}$ of the corresponding 6D toy densities.}
\label{fig:synth_demo}
\end{figure*}

\subsection{Additional experiments and illustration with 2D toy densities} \label{subsec:expextratwod}

We first analyse how the noise level of the log-likelihood evaluations affects the estimation accuracy. We use our three 2D toy densities and corrupt their exact log-likelihoods with additive zero mean i.i.d.~Gaussian noise with variance $\sigma_n^2\in\{1^2,2^2,5^2\}$ to obtain ``noisy evaluations''. We use initial designs with $b_0=10$ points. The integrals in \eiv{} and \miiqr{} are here numerically computed using a $50\times 50$ grid. We use batch size $b=4$ for all batch methods and we compare both the joint and greedy batch methods. The results are shown in Figure \ref{fig:synth_res2d}. We see that the principled \eiv{} and \miiqr{} methods have fairly similar overall performance and they clearly outperform the heuristic \maxv{} and \maxiqr{} methods. 
Unsurprisingly, the noisy evaluations require more iterations. 

\begin{figure*}[!htb]
\centering
\includegraphics[width=0.99\textwidth]{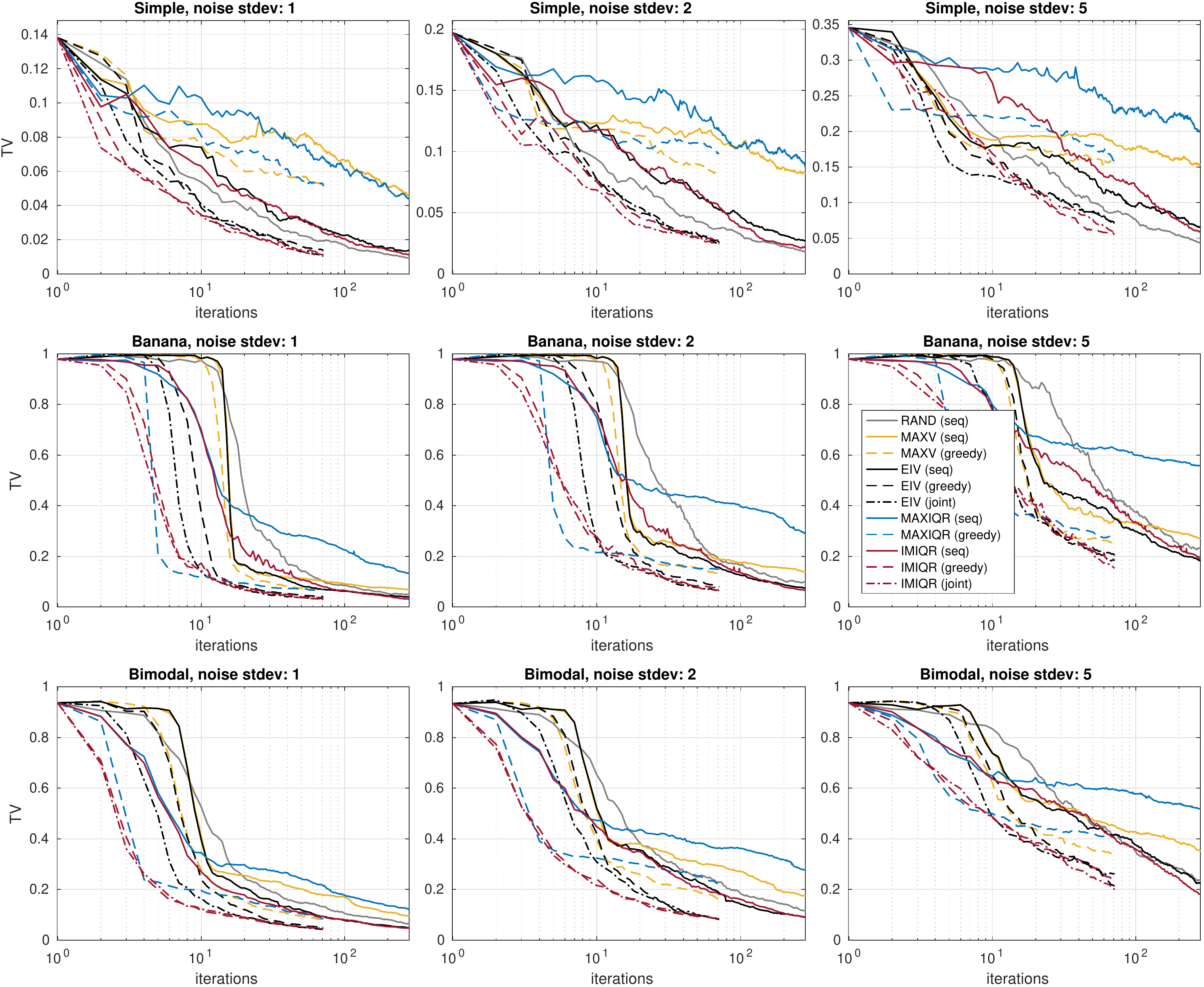}
\caption{Results for the 2D toy densities with various noise levels. The rows correspond to the test cases and the columns different noise levels of the log-likelihood evaluations. The lines show the average TV over $50$ repeated simulations. Note that x-axis is on log-scale and the maximum number of evaluations is here $t=290$. 
The batch size of all batch methods is $b=4$.
}
\label{fig:synth_res2d}
\end{figure*}

Figure~\ref{fig:acq_points_demo} shows the design locations of the 2D Banana example for \maxv{}, \eiv{} and \maxiqr{}.
Figures~\ref{fig:acq_points_demo2_part1} and \ref{fig:acq_points_demo2_part2} show the design locations for the 2D Bimodal example. The Bimodal example shows similar general observations as the Banana example. Figure~\ref{fig:acq_points_demo2_part1} shows that all \miiqr{} methods produce similar designs.

\begin{figure*}[htb]
\centering
\includegraphics[width=0.99\textwidth]{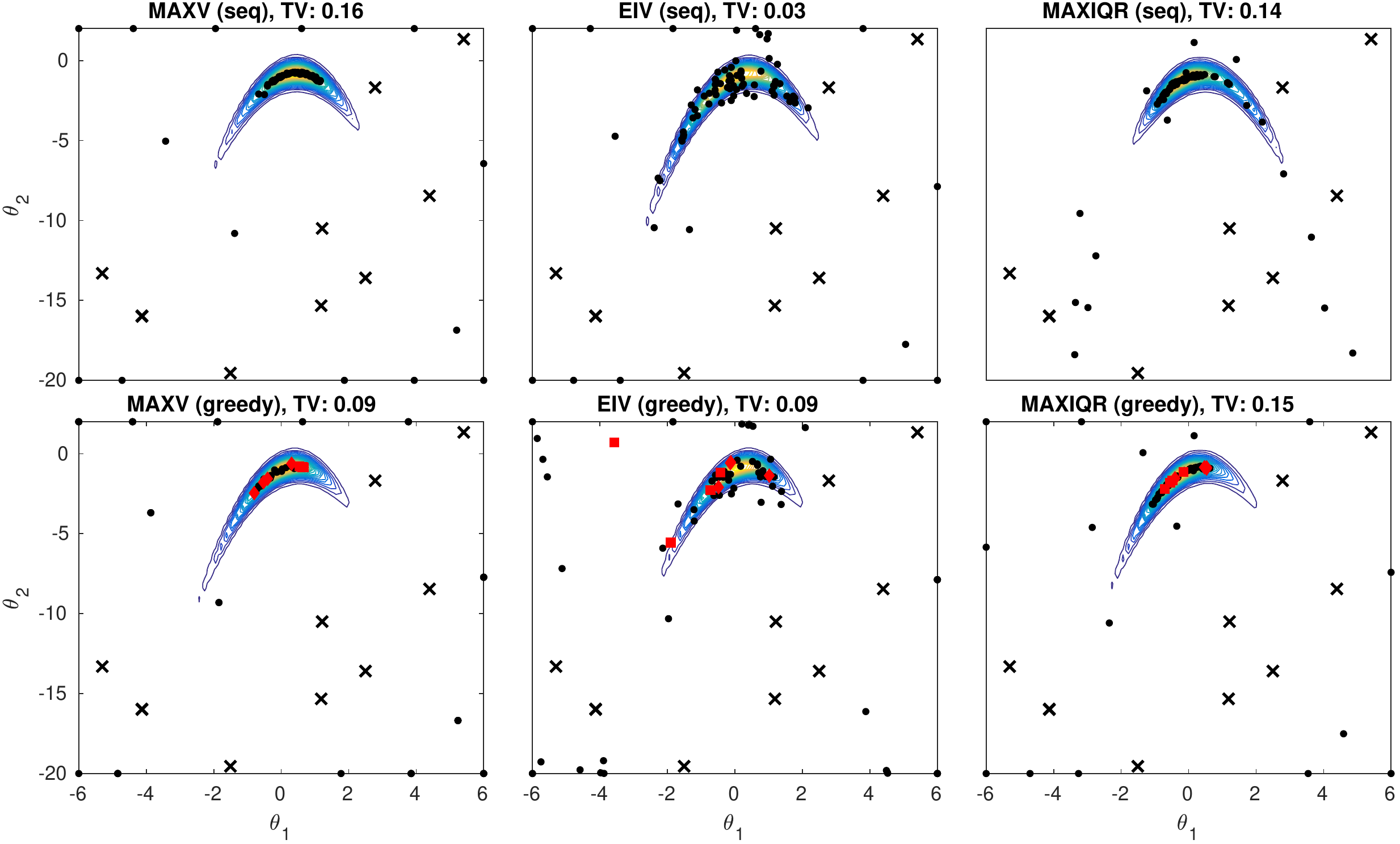}
\caption{Illustration of the designs of \maxv{} (left column), \eiv{} (center), and \maxiqr{} (right). The results are shown after $90$ noisy log-likelihood evaluations of the 2D Banana example with noise level $\sigma_n = 1$. The top row shows designs obtained with a sequential strategy ($b=1$), and the bottom row from a greedy batch-sequential strategy ($b=4$). The black crosses show the initial evaluations, black dots show obtained design points except for the last two batches, and the red squares and diamonds show the last two batches. 
}
\label{fig:acq_points_demo}
\end{figure*}

\begin{figure*}[htbp]
\centering
\includegraphics[width=0.99\textwidth]{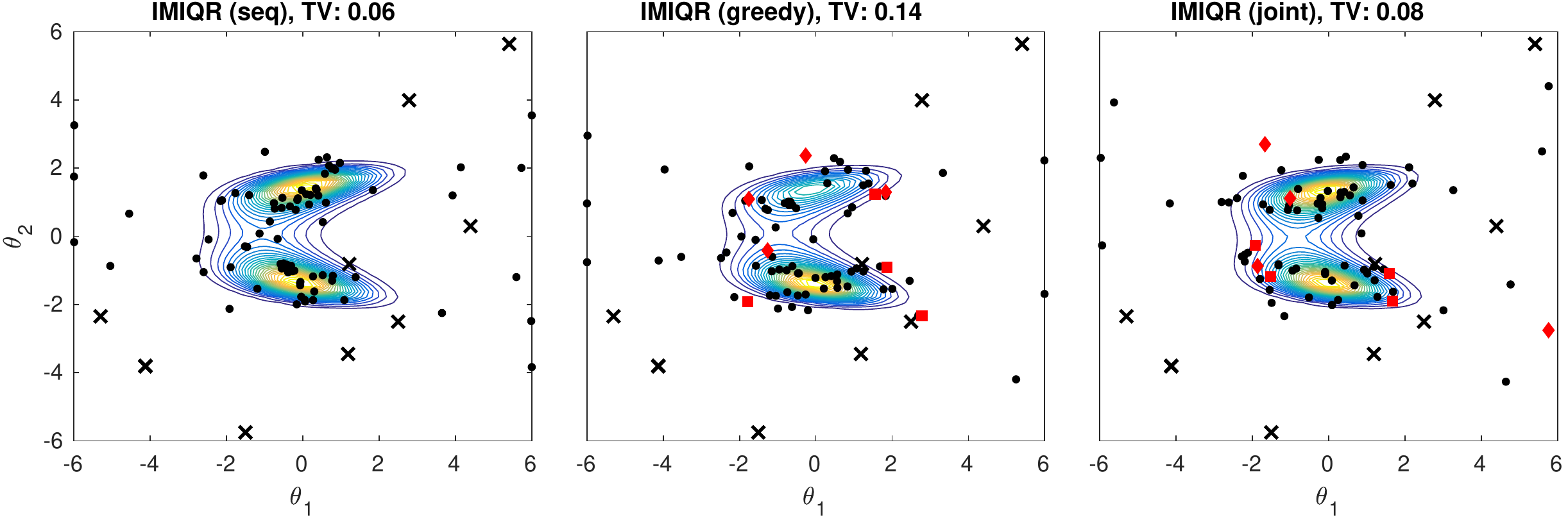}
\caption{Illustration of the design locations for the Bimodal example with \miiqr{} methods. See caption of Figure~\ref{fig:acq_points_demo} for details. }
\label{fig:acq_points_demo2_part1}
\end{figure*}

\begin{figure*}[htbp]
\centering
\includegraphics[width=0.99\textwidth]{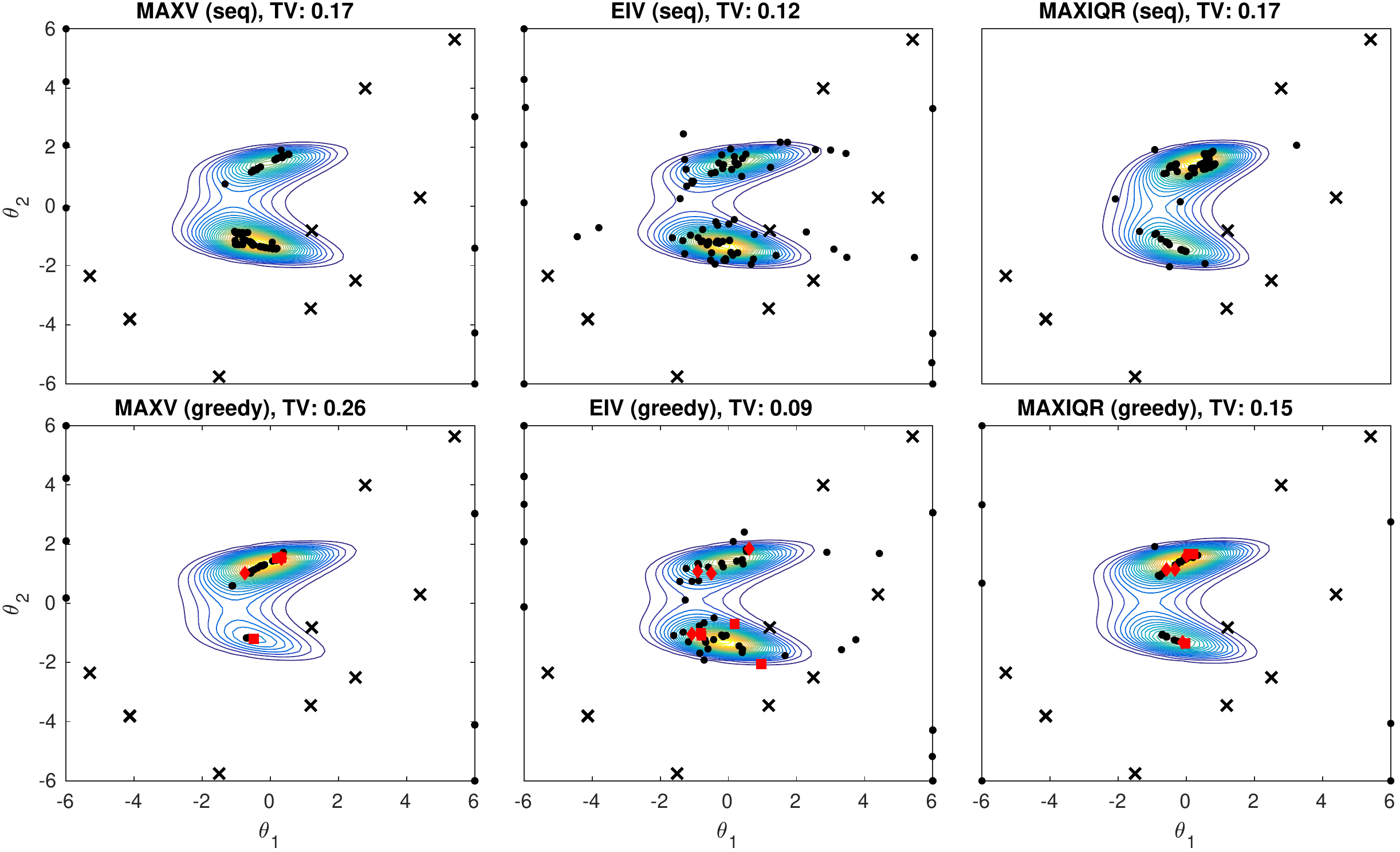}
\caption{Illustration of the design locations for the Bimodal example with \maxv{}, \eiv{} and \maxiqr{} methods. See caption of Figure~\ref{fig:acq_points_demo} for details. }
\label{fig:acq_points_demo2_part2}
\end{figure*}

We also investigated how the accuracy of the greedy batch \eiv{} and \miiqr{} methods scales as a function of the batch size using the three 2D toy densities. Figure~\ref{fig:synth_res_batch2d} shows these results.  Figure~\ref{fig:synth_res_batch2d_2} further shows the same experimental results as Figure~\ref{fig:synth_res_batch2d} but plotted as a function of the total evaluations. That is, the results obtained when the batch-sequential strategies are used as if they were sequential strategies so that the computations are done in a sequential manner. This demonstrates the penalty of not being able to use the unknown outputs of the pending simulations when they would in fact be available. The main observation is that the \miiqr{} greedy batch strategy produces batches that better mimic the sequential decisions than the corresponding \eiv{} batch strategy which shows as a faster convergence speed of \miiqr{}. 

\begin{figure*}[htbp]
\centering
\includegraphics[width=1\textwidth]{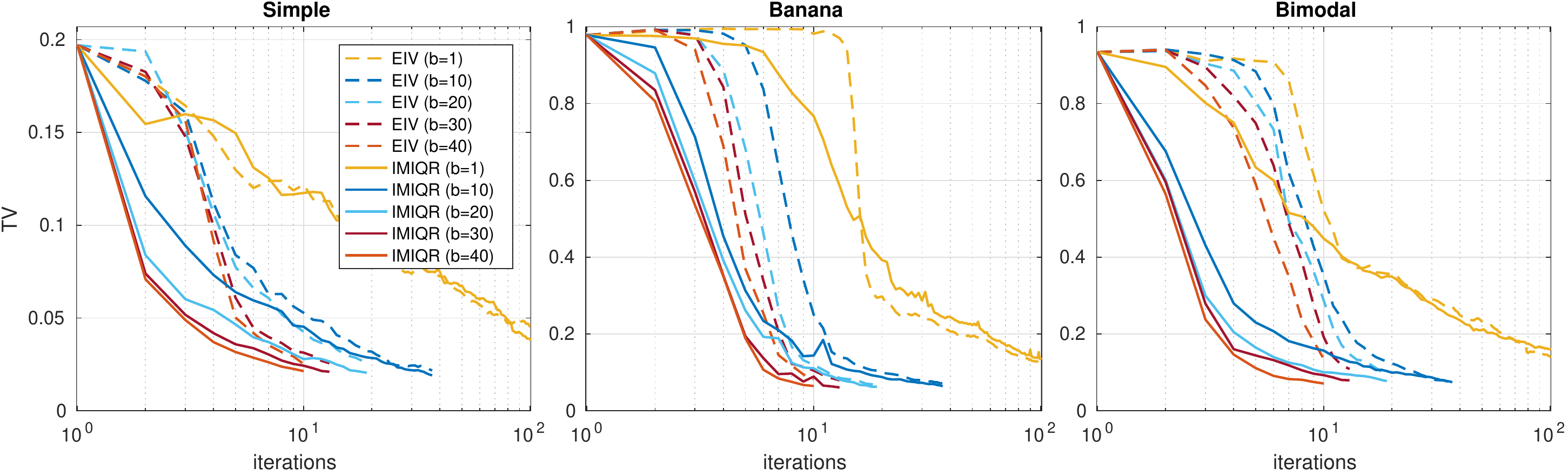}
\caption{Results with the greedy batch strategies with varying batch sizes $b$ for the 2D toy examples.}
\label{fig:synth_res_batch2d}
\end{figure*}

\begin{figure*}[htbp]
\centering
\includegraphics[width=1\textwidth]{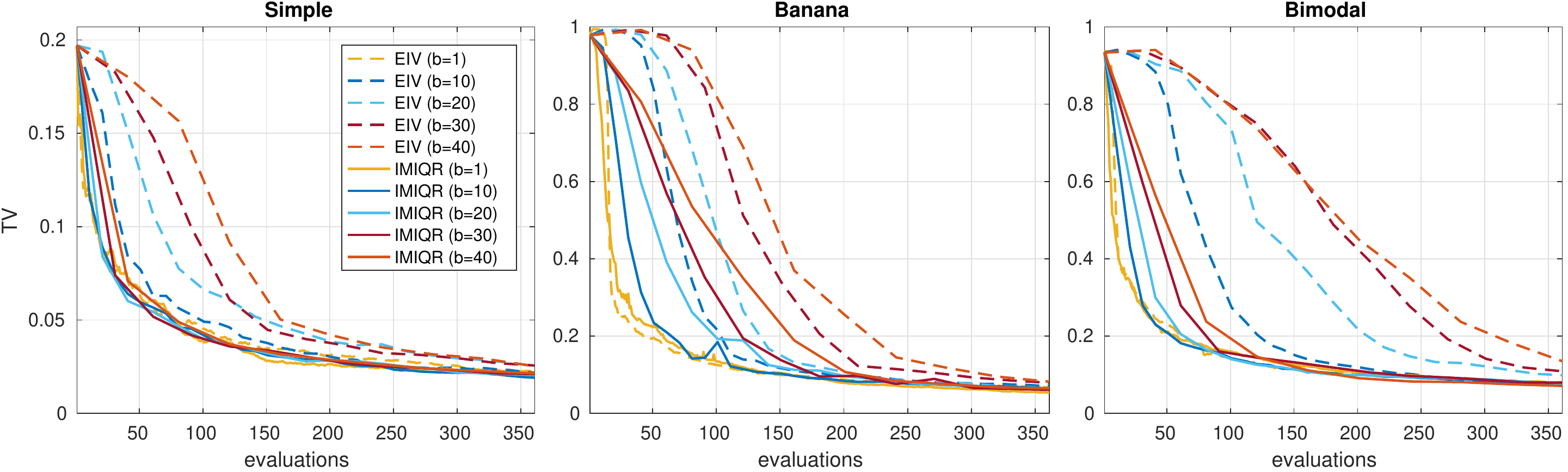}
\caption{Results with the greedy batch strategies with varying batch sizes $b$ for the 2D toy examples. The x-axis shows the total evaluations of the log-likelihood (and not the iterations in log-scale as in Figure~\ref{fig:synth_res_batch2d}).}
\label{fig:synth_res_batch2d_2}
\end{figure*}

\subsection{Parallellisation of the synthetic likelihood method with IMIQR} \label{subsec:N_parallellisation}

We demonstrate that it is useful to parallellise our algorithm in the SL case with respect to $b$. This justifies our experiments in Section \ref{sec:real_example}. We use the same set-up for Ricker and g-and-k models as in the main text and the set-up for Lorenz model is as in Section \ref{subsec:lorenz}. 
We consider a scenario where $1000$ computer cores are available
and we compare three different combinations of $N$ and $b$ that all use $1000$ simulations per batch. In addition, we consider a baseline where $N=100$ and $b=1$. We use \miiqr{} design criterion. 

The results in Figure~\ref{fig:N} show that $b=10$ gives the fastest convergence speed. This is not surprising since, intuitively, obtaining $10$ noisy log-likelihood evaluations gives more information on the shape of the posterior than a single, although less noisy, evaluation. Also, computing log-likelihood values accurately near the boundaries of the parameter space wastes computational resources because a noisy evaluation is often enough to effectively rule out such regions. 
A surprising aspect is that $N=100$ produces faster convergence than the corresponding runs with $N=1000$. The likely reason for this seemingly counterintuitive behaviour is that GP modelling becomes easier with the noisier evaluations which helps to faster locate the modal region. When $N=1000$, the \miiqr{} produced slightly more evaluations near the boundary areas as compared to $N=100$ leading to slower initial convergence speed. 
On the other hand, especially in the Lorenz case, the more accurate log-likelihood computations can be useful at the later iterations when the modal area of the posterior have been roughly located.

\begin{figure*}[htbp]
\centering
\includegraphics[width=0.99\textwidth]{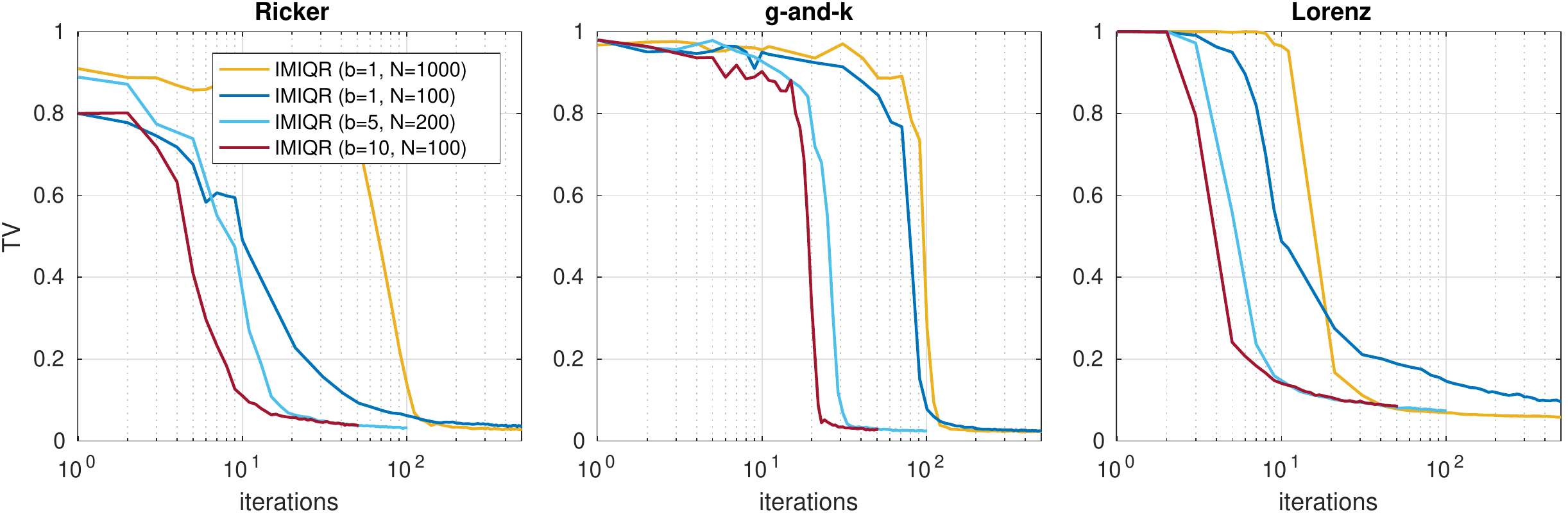}
\caption{The effect of the number of repeated simulations $N$ and the batch size $b$ on the convergence speed in the SL case. The details of the Lorenz model can be found in Section \ref{subsec:lorenz}.}
\label{fig:N}
\end{figure*}

\subsection{Experiments with Lorenz model} \label{subsec:lorenz}

As an additional real-world test case, we consider a modified version of the well-known Lorenz weather model. We briefly describe this model below, for more details see \citet{Thomas2018} and the references therein. In the model it is assumed that weather stations measure a high-dimensional time-series of slow weather variables $x_{k}^{(t)}$, $k=1,\ldots,40$, whose dynamics is described by a
coupled stochastic differential equation (SDE)
\begin{align}
\frac{\ud x_k^{(t)}}{\ud t} &= -x_{k-1}^{(t)}(x_{k-2}^{(t)} - x_{k+1}^{(t)}) - x_{k}^{(t)} + 10 - g(x_{k}^{(t)},\Btheta) + \eta_k^{(t)}, 
%
\label{eq:g_lorenz}
\end{align}
for $k=1,\ldots,40$ and where $g(x_{k}^{(t)},\Btheta) = \theta_1 + \theta_2 x_{k}^{(t)}$ and where $x_{k}^{(t)}$ are cyclic so that $x_{0}^{(t)} = x_{40}^{(t)}$ and $x_{-1}^{(t)} = x_{39}^{(t)}$. 
The initial states of the weather variables $x_{k}^{(0)}$, $k=1,\ldots,40$ are assumed known and the time interval $t\in[0,4]$ considered here corresponds to $20$ days. The function $g$ in \eqref{eq:g_lorenz} models the net effect of the fast weather variables on the observable slow weather variables $x_{k}^{(t)}$, $k=1,\ldots,40$ and $\eta_k^{(t)}$ is a stochastic forcing term describing the uncertainty due to the forcing of the fast weather variables. 
As in \citet{Thomas2018}, the time interval is discretised to $160$ equidistant intervals producing time step $\Delta t = 0.025$ and the SDEs are then solved using 4th order Runge-Kutta method. In this discretised setting, the forcing term is assumed to follow the first-order autoregressive process
\begin{equation}
\eta_k^{(t+\Delta t)} = \phi \eta_k^{(t)} + \sqrt{1-\phi^2}\epsilon^{(t)}
\end{equation}
for $k=1,\ldots,40$ and $t=0,\Delta t, 2\Delta t,\ldots,160\Delta t$, and where $\epsilon^{(t)}$ are i.i.d.~standard Gaussian, $\eta_k^{(0)} = \sqrt{1-\phi^2}\epsilon^{(0)}$ and $\phi=0.4$.

We need to compute the posterior distribution of the parameters $\Btheta=(\theta_1,\theta_2)$ given the slow weather variables $x_{k}^{(t)}$, $k=1,\ldots,40$ measured over $20$ days. 
We use $\Btheta \sim \Unif([0,5]\times[0,0.5])$ which is wider than in \citet{Thomas2018}, to make the inference task more challenging, and the true parameter to generate the observed data is $\Btheta_{\text{true}} = (2.0,0.1)$.
We use the six summary statistics suggested by \citet{Hakkarainen2012} and used by \citet{Thomas2018} to be in line with previous work although it was recently shown by \citet{Dinev2018} that learning them from data can improve the estimation accuracy. 
We use $b_0=10$ and an additional budget of $400$ SL evaluations with $N=100$. 

The results in Figure~\ref{fig:lorenz} again show that the best approximations are obtained with \miiqr{} and its greedy batch variant with $b=5$, which converges five times faster than the sequential \miiqr{}. 
This time \maxiqr{}, despite its exploitative behaviour, and its batch version with $b=5$, both work reasonably well although they produce slightly worse and more variable posterior approximations than \miiqr{}. \eiv{} and \maxv{} perform again very poorly because they mostly evaluate near the boundaries where the noise variance of log-likelihood is large although these evaluations are uninformative for estimating the likelihood in its modal area. 

Overall, the TV values with Lorenz model are slighty worse than in the corresponding synthetic 2D examples although we use $N=100$ so that $\sigma_n\lesssim 1$ in the modal area of the posterior. The probable reason is that the Gaussian assumption does not hold here as well as in other cases as suggested by Figure~\ref{fig:sl_gaussianity}. The third panel of Figure~\ref{fig:N} suggest that larger $N$ would lead to better approximation, likely because large $N$ makes the density of log-SL more peaked and also more Gaussian. 
Nevertheless, we conclude that the approximations obtained by \miiqr{} are reasonable and the convergence speed is fast since only a few hundred SL evaluations are needed. 

\begin{figure*}[htbp]
\centering
\includegraphics[width=0.99\textwidth]{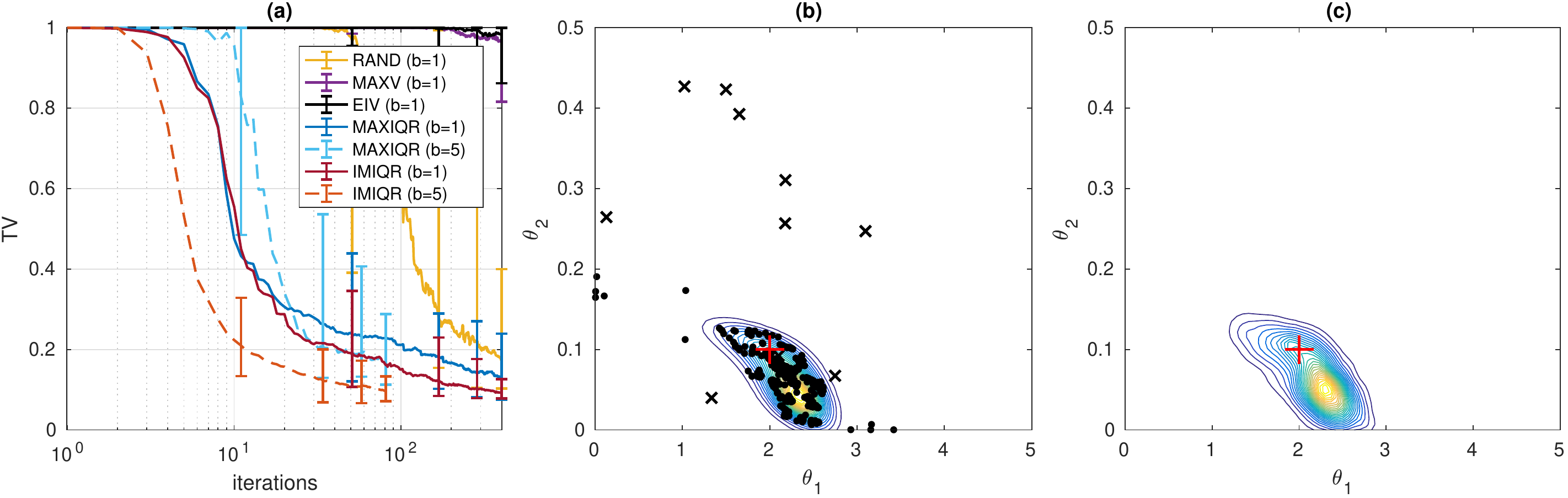}
\caption{Results for the Lorenz model. (a) Median TV and $90$\% variability over $50$ repeated runs of the algorithms. (b) A typical example of the estimated posterior density obtained using the greedy batch \miiqr{} strategy. The black crosses show the initial design and the dots the design points. (c) Exact SL posterior computed using SL-MCMC for comparison. The red plus sign shows the true parameter.}
\label{fig:lorenz}
\end{figure*}

\end{document}